\documentclass[11pt]{article}
   
\usepackage{xargs}                      
\usepackage[pdftex,dvipsnames]{xcolor}  

\usepackage[colorinlistoftodos,prependcaption,textsize=tiny,disable]{todonotes}

\newcommandx{\mnote}[2][1=]{\todo[linecolor=red,backgroundcolor=red!25,bordercolor=red,#1]{M: #2}}
\newcommandx{\unsure}[2][1=]{\todo[linecolor=blue,backgroundcolor=blue!25,bordercolor=blue,#1]{#2}}
\newcommandx{\knote}[2][1=]{\todo[linecolor=OliveGreen,backgroundcolor=OliveGreen!25,bordercolor=OliveGreen,#1]{#2}}
\newcommandx{\improvement}[2][1=]{\todo[linecolor=Plum,backgroundcolor=Plum!25,bordercolor=Plum,#1]{#2}}
\newcommandx{\cut}[2][1=]{\todo[linecolor=Plum,backgroundcolor=Plum!25,bordercolor=Plum,#1]{Cut: #2}}
\newcommandx{\thiswillnotshow}[2][1=]{\todo[disable,#1]{#2}}

\usepackage[bookmarks,colorlinks,breaklinks]{hyperref}  
\hypersetup{linkcolor=blue,citecolor=blue,filecolor=blue,urlcolor=blue} 
\usepackage{fullpage,appendix}
\usepackage{amsmath,amsfonts,amsthm,amssymb,xspace,bm}
\usepackage{dsfont}
\usepackage{algorithm}
\usepackage[noend]{algpseudocode}

\usepackage{tikz}
\newcommand{\newref}[2][]{\hyperref[#2]{#1~\ref*{#2}}}
\renewcommand{\eqref}[1]{\hyperref[#1]{(\ref*{#1})}}
\numberwithin{equation}{section}

\newcommand{\sref}[1]{\newref[Section]{#1}}

\newcommand{\tref}[1]{\newref[Theorem]{#1}}
\newcommand{\lref}[1]{\newref[Lemma]{#1}}

\newcommand{\cref}[1]{\newref[Corollary]{#1}}

\newcommand{\eref}[1]{\newref[Equation]{#1}}

\newcommand{\clref}[1]{\newref[Claim]{#1}}

\newcommand{\bmx}{\bm{\widetilde{x}}}
\newcommand{\bmw}{\bm{\widetilde{w}}}

\theoremstyle{plain}
\newtheorem{theorem}{Theorem}[section]
\newtheorem{lemma}[theorem]{Lemma}
\newtheorem{claim}[theorem]{Claim}

\newtheorem{corollary}[theorem]{Corollary}

\newtheorem{definition}[theorem]{Definition}
\newtheorem{fact}[theorem]{Fact}
\theoremstyle{definition}

\newcommand{\bkets}[1]{\left(#1\right)}
\newcommand{\sbkets}[1]{\left[#1\right]}

\newcommand{\abs}[1]{\left|#1\right|}

\DeclareMathOperator*{\pr}{\mathsf{Pr}}

\DeclareMathOperator*{\ex}{\mathbb{E}}

\newcommand{\iprod}[2]{#1{\mathbf \cdot} #2}
\renewcommand{\cal}[1]{\mathcal{#1}}

\def\inpw#1,#2{\langle #1, #2\rangle}

\newcommand{\reals}{\mathbb{R}}

\newcommand{\R}{\reals}

\newcommand{\dpm}{\{1,-1\}}

\newcommand{\note}[1]{\marginpar{\tiny *note in TeX*}}
\newcommand{\ignore}[1]{}

\newcommand{\calD}{{\cal D}}

\renewcommand{\phi}{\varphi}
\renewcommand{\epsilon}{\varepsilon}

\newcommand{\D}{{{\cal D}}}

\newcommand{\spty}{\lambda}

\author{Adam R. Klivans\thanks{Department of Computer Science,
    University of Texas at Austin, \texttt{klivans@cs.utexas.edu}. Part of this work was done while the authors were visiting the Simons Institute for Theoretical Computer Science, Berkeley.}
  \and Raghu Meka\thanks{Department of Computer Science, UCLA, \texttt{raghum@cs.ucla.edu}. Supported by NSF Career Award CCF-1553605.}}
\title{Learning Graphical Models Using Multiplicative Weights}
\date{}

\begin{document}

\maketitle
	
	\thispagestyle{empty}
	\begin{abstract}
          We give a simple, multiplicative-weight update algorithm for
          learning undirected graphical models or Markov random fields
          (MRFs).  The approach is new, and for the well-studied case
          of Ising models or Boltzmann machines\mnote{Removed bounded
            degree here.}, we obtain an algorithm that uses a nearly
          optimal number of samples and has running time
          $\tilde{O}(n^2)$ (where $n$ is the dimension), subsuming and improving on all prior work.
          Additionally, we give the first efficient algorithm
          for learning Ising models over {\em non-binary} alphabets. 
          
          Our main application is an algorithm for learning the structure of $t$-wise MRFs with nearly-optimal sample complexity (up to polynomial losses in necessary terms that depend on the weights) and running time that is $n^{O(t)}$. In addition, given $n^{O(t)}$ samples, we can also learn the parameters of the model and generate a hypothesis that is close in statistical distance to the true MRF. 
          \ignore{Our main application is an algorithm for learning $t$-wise
          MRFs with sample complexity
          $n^{O(t)}$ (where we suppress some necessary terms that
          depend on the weights) and running time $n^{O(t)}$.  Our
          algorithm both reconstructs the underlying graph and
          generates a hypothesis that is close in statistical distance
          to the underlying probability distribution.} All prior work
          runs in time $n^{\Omega(d)}$ for graphs of bounded degree
          $d$ and does not generate a hypothesis close in statistical
          distance even for $t = 3$.  We observe that our runtime has
          the correct dependence on $n$ and $t$ assuming the hardness
          of learning sparse parities with noise.

          Our algorithm-- the {\em Sparsitron}-- is easy to implement
          (has only one parameter) and holds in the on-line setting.
          Its analysis applies a regret bound from Freund and
          Schapire's classic Hedge algorithm.  It also gives the first solution to the
          problem of learning sparse Generalized Linear Models (GLMs).
	\end{abstract}
	
	\newpage
	\setcounter{page}{1}
	\section{Introduction}
      Undirected graphical models or \emph{Markov random fields}
        (MRFs) are one of the most well-studied and influential
        probabilsitic models with applications to a wide range of
        scientific disciplines
        \cite{KollerFriedman,Lauritzen,MRS13,Hinton,KFL,Sal09,Clifford,JEMF}.
        Here we focus on binary undirected graphical models which are
        distributions $(Z_1,\ldots,Z_n)$ on $\dpm^n$ with an
        associated undirected graph $G$ - known as the
        \emph{dependency graph} - on $n$ vertices where each $Z_i$
        conditioned on the values of $(Z_j: j \text{ adjacent to $i$
          in $G$})$ is independent of the remaining variables.

\ignore{variables $Z_j$ once conditioned on 
  Let $Z_{1},\ldots,Z_{n}$ be a random draw from a probability
  distribution $P$ on $\{-1,1\}^n$ that is strictly positive.  Then
  $P$ can be represented as an {\em undirected graphical model} or {\em Markov
    Random Field} (MRF) $G$, an undirected graph on $n$ vertices where
  $Z_{i}$ is independent of $Z_{j}$ conditioned on the values of the
  other variables {\em if and only if} the edge $(i,j)$ does not appear in
  $G$.}

Developing efficient algorithms for inferring the structure of the
underlying graph $G$ from random samples from $\calD$ is a central
problem in machine learning, statistics, physics, and computer science
\cite{AKN06,KargerSrebro,WRL,BMS,NBSS,TandonRavikumar} and has attracted considerable attention from researchers in
these fields.  A famous early example of such an algorithmic result is
due to Chow and Liu from 1968 \cite{ChowLiu} who gave an efficient algorithm for
learning graphical models where the underlying graph is a tree.
Subsequent work considered generalizations of trees \cite{ATHW} and graphs under
various strong assumptions (e.g., restricted strong convexity
\cite{NRWY} or correlation decay \cite{BMS,RSS}).

  The current frontier of MRF learning has focused on the {\em
    Ising model} (also known as \emph{Boltzmann machines}) on \emph{bounded-degree graphs}, a special class of
  graphical models with only \emph{pairwise interactions} and each vertex having degree at most $d$ in the underlying dependency graph. We refer to \cite{Bre} for an extensive historical overview of the problem. Two important works of note are due to Bresler \cite{Bre}
  and \cite{VMLC} \mnote{I prefer not to use et al. notation.} who learn Ising models on bounded degree graphs. \cut{without any
  assumptions on the graph structure other than necessary non-degeneracy
  conditions.}

  Bresler's algorithm is a combinatorial (greedy) approach that runs
  in time $\tilde{O}(n^2)$ but requires doubly exponential in $d$
  many samples from the distribution (only singly exponential is
  necessary).  \cite{VMLC} use machinery from convex programming to
  achieve nearly optimal sample complexity for learning Ising models
  with zero external field and with running
  time $\tilde{O}(n^4)$.  Neither of these results are proved to hold over non-binary
  alphabets or for general MRFs.

\subsection{Our Results}

The main contribution of this paper is a simple, multiplicative-weight
update algorithm for learning MRFs.  Using our
algorithm we obtain the following new results:

\begin{itemize}

\item An efficient online algorithm for learning Ising models on arbitrary
  graphs with nearly optimal sample complexity and running time
  $\tilde{O}(n^2)$ per example (precise statements can be
  found in \sref{sec:ising}).  In particular, for bounded degree
  graphs we achieve a run-time of $\tilde{O}(n^2)$ with nearly optimal sample
  complexity.   This subsumes and improves all prior work
  including the above mentioned results of Bresler \cite{Bre} and \cite{VMLC}.
  Our algorithm is the first that works even for {\em unbounded-degree} graphs
  as long as the $\ell_1$ norm of the weight vector of each
  neighborhood is bounded, a condition necessary for
  efficiency (see discussion following \cref{cor:ising}).

\item An algorithm for learning the dependency graph of binary $t$-wise Markov random fields
  with nearly optimal sample complexity and run-time $n^{O(t)}$ (precise statements can be
  found in \sref{sec:mrf}).  Moreover, given access to roughly $n^{O(t)}$ samples (suppressing necessary terms depending on the weights), we can also reconstruct the parameters of the model and output a $t$-wise MRF that gives a point-wise approximation to the original distribution. 

\ignore{An algorithm for learning general $t$-wise Markov random fields
  with sample complexity roughly $M \approx n^{O(t)}$ and running time
  $O(M \cdot n^{t})$ (we suppress some necessary terms that
  depend on the weights; precise statements can be
  found in Section \ref{sec:mrf}); precise statements can be
  found in \sref{sec:mrf}).  We reconstruct both the underlying graph and
  output a function $f$ that generates a distribution arbitrarily
  close in statistical distance.}

\end{itemize}
 As far as we are aware, these are the {\em first efficient algorithms} for
 learning higher-order MRFs.  All previous work on learning general
  $t$-wise MRFs
  runs in time $n^{\Omega(d)}$ (where $d$ is the underlying degree of
  the graph) and does not output a function $f$ that
  can generate an approximation to the distribution in statistical
  distance, {\em even for the special case of} $t=3$.  We give
  evidence that the $n^{O(t)}$ dependence in our running time is
  nearly optimal by applying a simple reduction from the problem of learning sparse
  parities with noise on $t$ variables to learning $t$-wise MRFs due
  to Bresler, Gamarnik, and Shah \cite{BGS}
  (learning sparse parities with noise is a notoriously difficult
  challenge in theoretical computer science). Bresler \cite{Bre}
  observed that even for the simplest possible Ising model where the
  graph has a single edge, beating $O(n^2)$ run-time corresponds to
  fast algorithms for the well-studied {\em light bulb} problem
  \cite{Valiant88}, for which the best known algorithm runs in time
  $O(n^{1.62})$ \cite{Valiant15}. 

  Moreover, our algorithm is easy to implement, has only one tunable parameter, and works in an on-line
  fashion.  The algorithm-- the {\em Sparsitron}-- solves the problem
  of learning a sparse Generalized Linear Model.  That is, given
  examples $(X,Y) \in [-1,1]^n \times [0,1]$ drawn from a distribution
  ${\cal D}$ with the property that $\ex[Y | X = x] = \sigma(w \cdot x)$
  for some monotonic, Lipschitz $\sigma$ and unknown $w$ with $\|w\|_1 \leq \spty$, the
  {\em Sparsitron} efficiently outputs a $w'$ such that $\sigma(w' \cdot x)$ is
  close to $\sigma(w \cdot x)$ in \emph{squared-loss} and has sample complexity $O(\spty^2 \log
  n)$.

  In an independent and concurrent work, Hamilton, Koehler, and Moitra
  \cite{HKM} generalized Bresler's approach to hold for both
  higher-order MRFs as well as MRFs over general (non-binary) alphabets.   For
  learning binary MRFs on bounded-degree---degree at most $d$---graphs, under
  the same non-degeneracy assumption taken by Hamilton et al.,\footnote{A previous version of this manuscript needed a
    slightly stronger non-degeneracy assumption.} we obtain sample complexity that is singly exponential in $d^t$, whereas theirs is doubly exponential in $d^t$ (both of our papers obtain sample complexity that depends only logarithmically on $n$, the number of vertices).
  
\ignore{   For the 
  non-degeneracy assumptions taken by Hamilton et
  al. \cite{HKM}\footnote{A previous version of this manuscript took a
    slightly stronger non-degeneracy assumption.}, we
  obtain sample complexity that is singly exponential in $d^{t}$
  whereas theirs is doubly exponential in $d^{t}$ (both of our papers
  obtain sample complexity that depends only logarithmically on $n$, the
  number of vertices).}

\ignore{

\begin{itemize}

\item Inferring the structure of graphical models a central problem in
  computer science, machine learning, and statistics.  In particular,
  learning distributions encoded by a binary Ising model has been
  intensely studied.
}

\subsection{Our Approach}

For a graph $G = (V,E)$ on $n$ vertices, let $C_t(G)$ denote all
cliques of size at most $t$ in $G$.  We use the Hammersley-Clifford
characterization of Markov random fields and define a binary $t$-wise
Markov random field on $G$ to be a distribution $\calD$ on $\dpm^n$
where $$\pr_{Z \sim \calD}[Z= z] \propto \exp\bkets{\sum_{I \in C_t(G)} \psi_I(z)},$$
and each $\psi_I:\R^n \to \R$ is a function that depends only on the
variables in $I$. 

For ease of exposition, we will continue with the case of $t=2$, the
Ising model, and subsequently describe the extension to larger values
of $t$.  Let $\sigma(z)$ denote the {\em sigmoid} function.  That is
$\sigma(z) = 1/1 + e^{-z}$.   Since $t=2$, we have

$$\pr\sbkets{Z = z} \propto \exp\bkets{\sum_{i \neq j \in [n]} A_{ij}
  z_i z_j + \sum_i \theta_i z_i}$$

for a weight matrix $A \in \R^{n \times n}$ and $\theta \in \R^n$; here, a weight $A_{ij} \neq 0$ if and only if $\{i,j\}$ is an edge in the underlying dependency graph. For
a node $Z_{i}$, it is easy to see that the probability $Z_{i} = -1$
conditioned on any setting of the remaining nodes to some value
$x \in \{-1,1\}^{[n] \setminus \{i\}}$ is equal to
$\sigma(w \cdot x + \theta)$ where $w \in \R^{[n] \setminus \{i\}}$, $w_j = -2 A_{ij}$, $\theta = -\theta_i$.

As such, if we set $X \equiv (Z_j: j \neq i)$ and $Y = (1-Z_i)/2$, then the conditional expectation of $Y$ given $X$ is {\em equal} to a sigmoid with an unknown weight vector $w$ and threshold $\theta_{i}$.  We
can now rephrase our original {\em unsupervised} learning task as the following {\em supervised} learning
problem: Given random examples $(X,Y)$ with conditional
mean function $\ex[Y | X = x] = \sigma(w \cdot x + \theta)$, recover $w$ and
$\theta$.

Learning a conditional mean function of the form $u(w \cdot x)$ with a
fixed, known \emph{transfer function} $u:\R \to \R$ is {\em precisely} the problem of learning a {\em
  Generalized Linear Model} or GLM and has been studied extensively in
machine learning.  The first provably efficient algorithm for learning
GLMs where $u$ is both monotone and Lipschitz was given by Kalai and
Sastry \cite{KalaiSastry09}, who called their algorithm the ``Isotron''.  Their result was
simplified, improved, and extended by Kakade, Kalai, Kanade, and Shamir \cite{KKKS} who introduced the
``GLMtron'' algorithm.

Notice that $\sigma(z)$ is both monotone and $1$-Lipschitz.
Therefore, directly applying the GLMtron in our setting will result in a
$w' $ and $\theta'$ such that
\begin{equation}\label{eq:intro1}
\ex[(\sigma(w' \cdot x + \theta') - \sigma(w \cdot x + \theta))^2]
\leq \epsilon.
\end{equation}

Unfortunately, the sample complexity of the GLMtron depends on
$\|w\|_2$, which results in sub-optimal bounds on sample complexity for our setting\footnote{GLMtron in our setting would require $\Omega(n)$ samples; we are aiming for an information-theoretically optimal logarithmic dependence in the dimension $n$.}.
We desire sample complexity dependent on $\|w\|_1$, essentially the
{\em sparsity} of $w$.  In addition, we need an {\em exact recovery}
algorithm.  That is, we need to ensure that $w'$ itself is close to
$w$ and not just that the \emph{$\ell_2$-error} as in
\eref{eq:intro1} is small. We address these two challenges next.

Our algorithm, the \emph{Sparsitron}, uses a multiplicative-weight update
rule for learning $w$, as opposed to the GLMtron or Isotron, both of
which use additive update rules.  This enables us to achieve
essentially optimal sample complexity.  The Sparsitron is simple to
describe (see Algorithm \ref{hedge}) and depends on only one parameter $\spty$, the
upper bound on the $\ell_1$-norm.  Its analysis only uses a regret bound from the classic
Hedge algorithm due to Freund and Schapire \cite{FS}.

Although the Sparsitron algorithm finds a vector $w' \in \R^n$ such that $\ex_X[(\sigma(w' \cdot X
+ \theta') - \sigma(w \cdot X + \theta))^2]$ is small, we still must
prove that $w'$ is actually close to $w$. Achieving such strong recovery guarantees for arbitrary distributions is typically a much harder problem (and can be provably hard in some cases for related problems \cite{FGKP, GR}). In our case, we exploit the nature of MRFs by a clean property of such distributions: Call a distribution $\calD$ on $\dpm^n$ \emph{$\delta$-unbiased} if each variable $Z_i$ is $1$ or $-1$ with probability at least $\delta$ conditioned on any setting of the other variables. It turns out that under conditions that are necessary for reconstruction, the distributions of MRFs are $\delta$-unbiased for a non-negligible $\delta$. We show that for such $\delta$-unbiased distributions achieving reasonably small $\ell_2$-error as in \eref{eq:intro1} implies that the recovered coefficient $w'$ is in fact close to $w$. 

\ignore{Here we use the fact that the non-degeneracy conditions of the MRF imply that the underlying distribution on $Z$ is $\delta$-unbiased (i.e., each variable $Z_i$ is $1$ or $-1$ with probability at least $\delta$ conditioned on any setting of the other variables).  For general distributions, this type of proper learning is often computationally intractable.}

To obtain our results for learning $t$-wise Markov random fields, we generalize the above approach to handle functions of the form $\sigma(p(x))$ where $p$ is a degree $t$ multilinear polynomial. Sparsitron can be straightforwardly extended to handle low-degree polynomials by \emph{linearizing} such polynomials (i.e., working in the $(n^t)$-dimensional space of coefficients). We then have to show that achieving small $\ell_2$-error - $\ex_X[(\sigma(p(X)) - \sigma(q(X)))^2] \ll 1$ - implies that the polynomials $p,q$ are close. This presents several additional technical challenges; still, in a self-contained proof, we show this holds whenever the underlying distribution is $\delta$-unbiased as is the case for MRFs. 


\subsection{Best-Experts Interpretation of Our Algorithm}

Our algorithm can be viewed as a surprisingly simple weighted voting
scheme (a.k.a. ``Best-Experts'' strategy) to uncover the underlying
graph structure $G = (\{v_1,\ldots,v_n\}, E)$ of a Markov random field.  Consider an Ising model
where for a fixed vertex $v_i$, we want to determine $v_i$'s
neighborhood and edge weights. Let $Z = (Z_1,\ldots,Z_n)$ denote random draws from the Ising model. 
\mnote{Many changes here ...}
\begin{itemize}

\item Initially, all vertices $v_{j} (j \neq i)$ could be neighbors.  We
create a vector of ``candidate'' neighbors of length $2n-2$ with
entries $(j,+)$ and $(j,-)$ for all $j \neq i$.  Intuitively, since
we do not know if node $v_j$ will be negatively or positively
correlated with $v_i$, we include two candidate neighbors, $(j,+), (j,-)$ to cover the two cases.  

\item At the outset, every candidate is equally likely to be a
  neighbor of $v_i$ and so receives an initial {\em weight} of
$1/(2n-2)$.  Now consider a random draw from the Ising model $Z =
(Z_{1},\ldots,Z_{n})$.  For each $j \neq i$ we view each $Z_{j}$
(and its negation -$Z_{j}$) as the {\em vote} of $(j,+)$ for the value $Z_i$ (respectively of $(j,-)$).
The overall {\em prediction} $p$ of our candidates is equal to a weighted
sum of their votes (we always assume the weights are non-negative and normalized appropriately). 

\item For a candidate neighbor $v_{j}$, let the {\em penalty} of the prediction $p$ (as motivated by the conditional mean function) 
  be equal to $\ell_j = (\sigma(-2 p) - (1-Z_i)/2)Z_{j}$. Each candidate $v_j$'s weight is simply multiplied by
  $\beta^{\ell_j}$ (for some suitably chosen \emph{learning rate} $\beta$\footnote{For our analysis, the learning rate can be set using standard techniques, e.g.,  $\beta = 1-\sqrt{\log n/T}$ when processing $T$ examples.}).  It is
  easy to see that candidates who predict $Z_i$ {\em correctly} will be
  penalized {\em less} than neighbors whose predictions are incorrect.

\end{itemize}

Remarkably, the weights of this algorithm will converge to the weights
of the underlying Ising model, and the rate of this
convergence is optimal.  Weights of vertices that are not neighbors of
$v_i$ will rapidly decay to zero.

For clarity, we present the updates for a single iteration of our Sparsitron algorithm applied to Ising model in Algorithm \ref{alg:ising}. The iterative nature of the algorithm is reminiscent of algorithms such as belief propagation and stochastic gradient descent that are commonly used in practice. Exploring connections with these algorithms (if any) is an intriguing question. 

\begin{algorithm}
\caption{Updates for \textsc{Sparsitron} applied to learning Ising models} \label{alg:ising}
Initialize $W_{ij}^+ = W_{ij}^- = 1/2(n-1)$ and $\hat{A}_{ij} = 0$ for $i \neq j$.\\
\textsc{Parameters:} \emph{Sparsity bound} $\lambda$. 
\begin{algorithmic}[1]
\For {each new example $(Z_1,\ldots,Z_n)$}:
\State Compute the current \emph{predictions}: $p_i = \sum_{j\neq i} \hat{A}_{ij} Z_j$ for all $i$.
\For {each $i \neq j$}
\State Compute the penalties: Set $\ell_{ij} = (\sigma(-2 p_i) - (1 - Z_i)/2) \cdot  Z_j$.
\State Update the weights: Set $W_{ij}^+ = W_{ij}^+ \cdot \beta^{\ell_{ij}}$; $W_{ij}^- = W_{ij}^- \cdot \beta^{-\ell_{ij}}$.
\EndFor
\For {each $i \neq j$}
\State Compute edge weights: $\hat{A}_{ij} = \frac{\spty}{\sum_{\ell \neq i} (W_{i\ell}^+ + W_{i\ell}^-)} \cdot \bkets{W_{ij}^+ - W_{ij}^-}.$
\EndFor
\EndFor
\end{algorithmic}
\end{algorithm}

\subsection{Organization}

We begin by describing the Sparsitron algorithm for learning sparse
generalized models and prove its correctness.  We then show, given a
hypothesis output by the Sparsitron, how to recover the underlying
weight vector {\em exactly} under $\delta$-unbiased distributions.
For ease of exposition, we begin by assuming that we are learning an
Ising model. 

We then describe how to handle the more general case of learning
$t$-wise MRFs.  This requires working with multilinear polynomials, and studying their behavior (especially, how small they can be) under $\delta$-unbiased distributions. 


\section{Preliminaries}
We will use the following notations and conventions.
\begin{itemize}
\item For a vector $x \in \R^n$, $x_{-i} \in \R^{[n] \setminus \{i\}}$ denotes $(x_j: j \neq i)$. 

\item We write multilinear polynomials $p:\R^n \to \R$ as  $p(x) = \sum_{I} \hat{p}(I) \prod_{i \in I} x_i$; in particular, $\hat{p}(I)$ denotes the coefficient of the monomial $\prod_{i \in I}x_i$ in the polynomial.  Let $\|p\|_1 = \sum_{I} |\hat{p}(I)|$.

\item For a multilinear polynomial $p:\R^n \to \R$, we let $\partial_i p(x) = \sum_{J: J \not \ni i} \hat{p}( J \cup \{i\}) \prod_{j \in J} x_j$ denote the partial derivative of $p$ with respect to $x_i$. Similarly, for $I \subseteq [n]$, let $\partial_I p(x) = \sum_{J: J \cap I = \emptyset} \widehat{p}(J \cup I) \prod_{j \in J} x_j$ denote the partial derivative of $p$ with respect to the variables $(x_i: i \in I)$. 

\item For a multilinear polynomial $p:\R^n \to \R$, we say $I \subseteq [n] $ is a \emph{maximal monomial} of $p$ if $\hat{p}(J) = 0$ for all $J \supset I$ (i.e., there is no non-zero monomial that strictly contains $I$). 
\end{itemize}


\section{Learning Sparse Generalized Linear Models}
We first describe our {\em Sparsitron} algorithm for learning sparse GLMs.  In the next section we show how to learn MRFs using this algorithm.  The main theorem of this section is the following:

\begin{theorem}\label{th:pconceptmwu}
Let $\calD$ be a distribution on $[-1,1]^n \times \{0,1\}$ where for
$(X,Y) \sim \calD$, $E[Y | X = x] = u(w \cdot x)$ for a non-decreasing $1$-Lipschitz function $u:\R \to [0,1]$. Suppose that $\|w\|_1 \leq \spty$ for a known $\spty \geq 0$. Then, there exists an algorithm that for all $\epsilon, \delta \in [0,1]$ given $T = O(\spty^2(\ln (n/\delta \epsilon))/\epsilon^2)$ independent examples from $\calD$, produces a vector $v \in \R^n$ such that with probability at least $1-\delta$,
\begin{equation}\label{eq:hedge1}
\ex_{(X,Y) \leftarrow \calD}[ (u(v \cdot X) - u(w \cdot X))^2] \leq \epsilon.
\end{equation}

The run-time of the algorithm is $O(n T)$. Moreover, the algorithm can be run in an online manner.
\end{theorem}
\begin{proof}
We assume without loss of generality that $w_i \geq 0$ for all $i$ and that $\|w\|_1 = \spty$; if not, we can map examples $(x,y)$ to $((x,-x,0),y)$ and work in the new space. For any vector $v \in \R^n$, define the \emph{risk} of $v$ $\epsilon(v) = \ex_{(X,Y) \sim \calD}[(u(v \cdot X) - u(w \cdot X))^2]$. Let $\bm{1}$ denote the all $1$'s vector. 

Our approach is to use the regret bound for the \emph{Hedge} algorithm of Freund and Schapire \cite{FS}.   Let $T \geq 1$, $\beta \in [0,1]$ be parameters to be chosen later and $M = C''' T \ln(1/\delta)/\epsilon^2$ for a constant $C'''$ to be chosen later. The algorithm is shown in Algorithm \ref{hedge}. The inputs to the algorithm are $T+M$ independent examples $(x^1,y^1,),\ldots,(x^T, y^T)$ and $(a^1,b^1),\ldots,(a^M,b^M)$ drawn from $\calD$. 

\begin{algorithm}
\caption{\textsc{Sparsitron} \label{hedge}}
\begin{algorithmic}[1]
\State Initialize $w^0 = \bm{1}/n$. 
\For {$t = 1,\ldots,T$}
\State Let $p^t = w^{t-1}/\|w^{t-1}\|_1$. 
\State Define $\ell^t \in \R^n$ by $\ell^t = (1/2) (\bm{1} + (u(\spty p^t \cdot x^t) - y^t) x^t)$.
\State Update the weight vectors $w^t$: for each $i \in [n]$, set $w^t_i = w^{t-1}_i \cdot \beta^{\ell^t_i}$.
\EndFor
\For{$t = 1,\ldots,T$} 
\State Compute the \emph{empirical risk} 
$$\hat{\epsilon}(\spty p^t) = (1/M) \sum_{j=1}^M \bkets{u(\spty p^t \cdot a^j) - b^j}^2.$$
\EndFor
\State \textsc{Return} $v = \spty p^j$ for $j = \arg\min_{t \in [T]} \hat{\epsilon}(\spty p^t)$.
\end{algorithmic}
\end{algorithm}
We add the $\bm{1}$ in Step 4 of Algorithm 2  to be consistent with \cite{FS} who work with loss vectors in $[0,1]^n$.

We next analyze our algorithm and show that for suitable parameters $\beta, T, M$, it achieves the guarantees of the theorem. We first show that the sum of the risks $\epsilon(\lambda p^1),\ldots,\epsilon(\lambda p^T)$ is small with high probability over the examples; the claim then follows by a simple Chernoff bound to argue that for $M$ sufficiently big, the empirical estimates of the risk, $\hat{\epsilon}(\spty p^1),\ldots,\hat{\epsilon}(\spty p^T)$ are close to the true risks. 


Observe that $\ell^t \in [0,1]^n$ and associate each $i = 1,\ldots,n$ with an expert and then apply the analysis of Freund and Schapire (c.f. \cite{FS}, Theorem 5). In particular, setting $\beta = 1/(1+ \sqrt{(\ln n)/T})$, we get that 

\begin{equation}\label{eq:hedge1} \sum_{t=1}^T \iprod{p^t}{\ell^t} \leq \min_{i \in [n]} \sum_{t=1}^T \ell^t_i + O(\sqrt{T \ln n} + (\ln n)).  \end{equation}

Let random variable $Q^t = \iprod{p^t}{\ell^t} - \iprod{(w/\spty)}{\ell^t}$. Note that $Q^t \in [-1,1]$. Let 
$$Z^t = Q^t - \ex_{(x^t,y^t)}[Q^t \mid (x^1,y^1),\ldots,(x^{t-1},y^{t-1})].$$  Then, $Z^1,\ldots,Z^T$ form a martingale difference sequence with respect to the sequence $(x^1,y^1),\ldots,(x^T,y^T)$ and are bounded between $[-2,2]$. Therefore, by Azuma-Hoeffding inequality for bounded martingale difference sequences, with probability at least $1-\delta$, we have $\abs{\sum_{t=1}^T Z^t} \leq O(\sqrt{T \ln(1/\delta)})$. Thus, with probability at least $1-\delta$, 
\begin{equation}\label{eq:hedgemds}
\sum_{t=1}^T \ex_{(x^t,y^t)}[Q^t \mid (x^1,y^1),\ldots,(x^{t-1},y^{t-1})] \leq \sum_{t=1}^T Q^t  + O( \sqrt{T \ln(1/\delta)}).
\end{equation}

Now, for a fixed $(x^1,y^1),\ldots,(x^{t-1},y^{t-1})$, taking expectation with respect to $(x^t,y^t)$, we have
\begin{align*}
\ex_{(x^t,y^t)}[Q^t \mid (x^1,y^1),\ldots,(x^{t-1},y^{t-1})] &= \ex_{(x^t,y^t)}\sbkets{\iprod{(p^t - (1/\spty) w)}{\ell^t}} \\
&= (1/2) \ex_{(x^t,y^t)}\sbkets{\iprod{(p^t - (1/\spty) w)}{(u( \iprod{\spty p^t}{x^t}) - y^t)x^t}}\\
&= (1/2\spty) \ex_{x^t}\sbkets{(\iprod{\spty p^t}{x^t} - \iprod{w}{x^t}) (u( \iprod{\spty p^t}{x^t}) - u(\iprod{w}{x^t}))}\\
&\geq (1/2\spty) \ex_{x^t} \sbkets{(u(\iprod{\spty p^t}{x^t}) - u(\iprod{w}{x^t}))^2}\\
&\text{ (for all $a,b \in \R$, $(a-b) (u(a) - u(b)) \geq (u(a) - u(b))^2$).}\\
&= (1/2\spty) \cdot \epsilon(\spty p^t).
\end{align*}

Therefore, for a fixed $(x^1,y^1),\ldots,(x^{t-1},y^{t-1})$, we have 
$$(1/2\spty ) \epsilon(\spty  p^t) \leq \ex_{(x^t,y^t)}[Q^t \mid (x^1,y^1),\ldots,(x^{t-1},y^{t-1})].$$ 

Combining the above with Equations \ref{eq:hedge1}, \ref{eq:hedgemds}, we get that with probability at least $1-\delta$,
\begin{align*}
(1/2\spty ) \sum_{t=1}^T \epsilon(\spty  p^t) &\leq \sum_{t=1}^T Q^t  + O( \sqrt{T \ln(1/\delta)})\\
&\leq \min_{i \in [n]} \sum_{t=1}^T \ell^t_i - \sum_{t=1}^T \iprod{(w/\lambda)}{\ell^t}  + O(\sqrt{T \ln n} + (\ln n)) +  O(\sqrt{T \ln(1/\delta)}).
\end{align*}

Now, let $L = \sum_{t=1}^T \ell^t$. Then, 
$$\min_{i \in [n]} \sum_{t=1}^T \ell^t_i - \sum_{t=1}^T (1/\spty ) w \cdot \ell^t = \min_{i \in [n]} L_i - (w/\spty ) \cdot L \leq 0,$$
where the last inequality follows as $\|w\|_1 = \spty $. Therefore, with probability at least $1-\delta$,
$$(1/2\spty ) \sum_{t=1}^T \epsilon(\spty  p^t) = O(\sqrt{T \ln(1/\delta)}) + O(\sqrt{T \ln n} + (\ln n)).$$
In particular, for $T > C'' \spty^2 (\ln(n/\delta))/\epsilon^2$ for a sufficiently big constant $C''$, with probability at least $1-\delta$,  
$$\min_{t\in [T]} \epsilon(\spty p^t) \leq O(\lambda) \cdot \frac{ \sqrt{T \ln(1/\delta)} + {\sqrt{T \ln n}} + \ln n}{T} \leq \epsilon/2.$$

Now set $M = C''' \ln(T/\delta)/\epsilon^2$ so that by a Chernoff-Hoeffding bound as in Fact \ref{lm:empapprox}, with probability at least $1-\delta$, for every $t \in [T]$, $\abs{\epsilon(\spty p^t) - \hat{\epsilon}(\spty p^t)} \leq \epsilon/4$. Therefore, with probability at least $1-2 \delta$, $\epsilon(v) \leq \epsilon/4 + \hat{\epsilon}(v) \leq \epsilon$. Note that the number of samples needed is $T + M = O(\spty^2 \ln(n/\epsilon \delta)/\epsilon^2)$. The theorem follows.

\end{proof}

\begin{fact}\label{lm:empapprox}
There exists a constant $C > 0$ such that the following holds. Let $v \in \R^n$ and let $(a^1,b^1),\ldots,(a^M,b^M)$ be independent examples from $\calD$. Then, for all $\rho, \gamma \geq 0$, and $M \geq C \ln(1/\rho)/\gamma^2$, 
$$\pr\sbkets{ \left|(1/M) \bkets{\sum_{j=1}^M (u(v \cdot a^j) - b^j)^2} - \epsilon(v)\right| \geq \gamma} \leq \rho.$$
\end{fact}

\section{Recovering affine functions from $\ell_2$ minimization}
In this section we show that running the Sparsitron algorithm with sufficiently low error parameter $\epsilon$ will result in an $\ell_{\infty}$ approximation to the unknown weight vector.  We will use this strong approximation to reconstruct the dependency graphs of Ising models as well as the edge weights. 

Our analysis relies on the following important definition:

\begin{definition}
A distribution $\calD$ on $\dpm^n$ is $\delta$-unbiased if for $X \sim \calD$, $i \in [n]$, and any partial assignment $x$ to $(X_j: j \neq i)$,
$$\min( \pr[X_i = 1 | X_{-i} = x], \pr[X_i = -1 | X_{-i} = x]) \geq \delta.$$
\end{definition}

We will use the following elementary property of sigmoid.

\begin{claim}\label{clm:sigmoidlb}
For $a, b \in R$,
 $$|\sigma(a) - \sigma(b)| \geq e^{-|a| - 3} \cdot \min\bkets{1, |a-b|}.$$
\end{claim}
\begin{proof}
Fix $a \in R$ and let $\gamma = \min(1,|a-b|)$. Then, since $\sigma$ is monotonic 
$$ |\sigma(a) - \sigma(b)| \geq \min(\sigma(a+\gamma) - \sigma(a), \sigma(a) - \sigma(a-\gamma)).$$
Now, it is easy to check by a case-analysis that for all $a,a' \in \R$,
$$|\sigma(a) - \sigma(a')| \geq \min(\sigma'(a), \sigma'(a')) \cdot |a - a'|.$$
Further, for any $t$, $\sigma'(t) = 1/(2 + e^t + e^{-t}) \geq e^{-|t|}/4$. Combining the above two, we get that
$$\sigma(a+ \gamma) - \sigma(a) \geq (1/4) \min(e^{-|a+\gamma|}, e^{-|a|}) \cdot \gamma \geq (1/4) e^{(-|a| - \gamma)} \gamma.$$
Similarly, we get
$$\sigma(a) - \sigma(a - \gamma) \geq 4 \min(e^{-|a- \gamma|}, e^{-|a|}) \cdot \gamma \geq (1/4) e^{(-|a| - \gamma)} \gamma.$$
The claim now follows by substituting $\gamma = \min(1,|a-b|)$ (and noting that $1/4 \geq e^{-2}$). 
\end{proof}

\ignore{
\begin{lemma}
If $a, b \in R$ with $|a - b | \geq \gamma$, then $|\sigma(a) - \sigma(b)| \geq (\gamma /4) \cdot e^{(-|a| - \gamma)}  $. 
\end{lemma}
\begin{proof}
Fix $a \in R$. Then, since $\sigma$ is monotonic 
$$\min_{b: |a-b| \geq \gamma} |\sigma(a) - \sigma(b)| \geq \min(\sigma(a+\gamma) - \sigma(a), \sigma(a) - \sigma(a-\gamma)).$$
Now, it is easy to check by a case-analysis that for all $a,a' \in \R$,
$$|\sigma(a) - \sigma(a')| \geq \min(\sigma'(a), \sigma'(a')) \cdot |a - a'|.$$
Further, for any $t$, $\sigma'(t) = 1/(2 + e^t + e^{-t}) \geq e^{-|t|}/4$. Combining the above two, we get that
$$\sigma(a+ \gamma) - \sigma(a) \geq (1/4) \min(e^{-|a+\gamma|}, e^{-|a|}) \cdot \gamma \geq (1/4) e^{(-|a| - \gamma)} \gamma.$$
Similarly, we get
$$\sigma(a) - \sigma(a - \gamma) \geq 4 \min(e^{-|a- \gamma|}, e^{-|a|}) \cdot \gamma \geq (1/4) e^{(-|a| - \gamma)} \gamma.$$
The claim now follows. 
\end{proof}
}

\begin{lemma}\label{lm:l2toexact}
Let $D$ be a $\delta$-unbiased distribution on $\dpm^n$. Suppose that for two vectors $v,w \in \R^n$ and $\alpha, \beta \in \R$, $\ex_{X \sim D}[ (\sigma(w \cdot X + \alpha) - \sigma(v \cdot X + \beta))^2] \leq \epsilon$ where $\epsilon < \delta \cdot \exp(-2\|w\|_1 - 2|\alpha| - 6)$. Then,
$$\|v-w\|_\infty \leq O(1) \cdot  e^{\|w\|_1 + |\alpha|} \cdot \sqrt{\epsilon/\delta}.$$
\end{lemma}

\begin{proof}
For brevity, let $p(x) = w \cdot x + \alpha$, and $q(x) = v \cdot x + \beta$. 
Fix an index $i \in [n]$ and let $X \sim D$. 

Now, for any $x \in \dpm^n$, by \clref{clm:sigmoidlb}, 
$$\abs{\sigma(p(x)) - \sigma(q(x))} \geq e^{-\|w\|_1 - |\alpha| - 3 } \cdot \min\bkets{1, \abs{p(x) - q(x)}}.$$

Let $x^{i,+} \in \dpm^n$ (respectively $x^{i,-}$) denote the vector obtained from $x$ by setting $x_i = 1$ (respectively $x_i = -1)$. Note that $p(x^{i,+}) - p(x^{i,-}) = 2 w_i$ and $q(x^{i,+}) - q(x^{i,-}) = 2 v_i$. Therefore,
$$ p(x^{i,+}) - q(x^{i,+}) - (p(x^{i,-}) - q(x^{i,-})) = 2 (w_i - v_i).$$
Thus, 
$$\max\bkets{\abs{p(x^{i,+}) - q(x^{i,+})}, \abs{p(x^{i,-}) - q(x^{i,-})}} \geq \abs{w_i - v_i}.$$
Therefore, for any fixing of $X_{-i}$, as $X$ is $\delta$-unbiased, 
$$\pr_{X_i | X_{-i}}\sbkets{ \abs{p(X) - q(X)} \geq \abs{w_i - v_i}} \geq \delta.$$
Hence, combining the above inequalities,
$$\epsilon \geq \ex_X\sbkets{(\sigma(p(X)) - \sigma(q(X)))^2}  \geq e^{-2 \|w\|_1 - 2 |\alpha| - 6} \cdot \delta \cdot \min\bkets{1, \abs{w_i - v_i}^2}.$$
As $\epsilon < e^{-2 \|w\|_1 - 2 |\alpha| - 6}  \delta$, the above inequality can only hold if $|w_i - v_i| < 1$ so that 
$$\abs{w_i - v_i} < e^{\|w\|_1 + |\alpha| + 3} \cdot \sqrt{\epsilon/\delta}.$$
The claim now follows.

\end{proof}

\section{Learning Ising Models}\label{sec:ising}
\begin{definition}
Let $A \in \R^{n \times n}$ be a \emph{weight matrix} and $\theta \in \R^n$ be a \emph{mean-field} vector. The associated $n$-variable \emph{Ising model} is a distribution $\calD(A, \theta)$ on $\dpm^n$ given by the condition 
$$\pr_{Z \leftarrow \calD(A,\theta)}\sbkets{Z = z} \propto \exp\bkets{\sum_{i \neq j \in [n]} A_{ij} z_i z_j + \sum_i \theta_i z_i}.$$
The \emph{dependency graph} of $\calD(A,\theta)$ is the graph $G$ formed by all pairs $\{i,j\}$ with $|A_{ij}| \neq 0$. We define $\spty(A,\theta) = \max_i (\sum_j |A_{ij}| + |\theta_i|)$ to be the \emph{width} of the model.
\end{definition}

We give a simple, sample-efficient, and online algorithm for recovering the parameters of an Ising model.
\begin{theorem}\label{th:ising}
Let $\calD(A,\theta)$ be an $n$-variable Ising model with width $\spty(A,\theta) \leq \spty$. There exists an algorithm that given $\spty$, $\epsilon,\rho  \in (0,1)$, and $N = O(\spty^2 \exp(O(\spty))/\epsilon^4) \cdot (\log (n/\rho \epsilon))$ independent samples $Z^1,\ldots,Z^N \leftarrow \calD(A,\theta)$ produces $\hat{A}$ such that with probability at least $1-\rho$, 
$$\|A - \hat{A}\|_\infty \leq \epsilon.$$
The run-time of the algorithm is $O(n^2 N)$. Moreover, the algorithm can be run in an online manner.
\end{theorem}
\begin{proof}
The starting point for our algorithm is the following observation. Let $Z \leftarrow \calD(A,\theta)$. Then, for any $i \in [n]$ and any $x \in \dpm^{[n] \setminus \{i\}}$, 
\begin{equation}\label{eq:condising}
\pr[Z_i = -1 | Z_{-i} = x] = \frac{1}{1 + \exp(2 \sum_{j \neq i} A_{ij} x_j + \theta_i)} = \sigma(w(i) \cdot x + \theta_i),
\end{equation}

where we define $w(i) \in R^{[n] \setminus \{i\}}$ with $w(i)_j = -2 A_{ij}$ for $j \neq i$. This allows us to use our Sparsitron algorithm for learning GLMs.

For simplicity, we describe our algorithm to infer the coefficients $A_{nj}$ for $j \neq n$; it extends straightforwardly to recover the weights $\{A_{ij}: j \neq i\}$ for each $i$.  Let $Z \leftarrow \calD(A,\theta)$ and let $X \equiv (Z_1,\ldots,Z_{n-1},1)$, and $Y = (1 - Z_n)/2$. Then, from the above we have that 
$$\ex[Y | X ] = \sigma(w(n) \cdot X),$$ 
where $w(n) \in \R^{n}$ with $w(n)_j = -2 A_{nj}$ for $j < n$, and $w(n)_n = \theta_i$. Note that $\|w(n)\|_1 \leq 2 \spty$. Further, $\sigma$ is a monotone $1$-Lipschitz function. Let $\gamma \in (0,1)$ be a parameter to be chosen later. We now apply the Sparsitron algorithm to compute a vector $v(n) \in \R^n$ so that with probability at least $1 - \rho/n^2$, 
\begin{equation}\label{eq2:ell2error}
\ex[(\sigma(w(n)\cdot X) - \sigma( v(n) \cdot X))^2] \leq \gamma.
\end{equation}

We set $\hat{A}_{nj} = - (v(n)_j)/2$ for $j < n$. We next argue that \eref{eq2:ell2error} in fact implies 
$\|w(n) - v(n)\|_\infty \ll 1$. To this end, we will use the following
easy fact (see e.g. Bresler \cite{Bre}):

\begin{fact}
For $Z \leftarrow \calD(A,\theta)$, $i \in [n]$, and any partial assignment $x$ to $Z_{-i}$, 
$$\min\bkets{\pr[Z_i = -1 | Z_{-i} = x], \pr[Z_i = 1 | Z_{-i} = x]} \geq (1/2) e^{-2 \spty(A,\theta)} \geq (1/2) e^{-2\spty}.$$
\end{fact}
That is, the distribution $Z$ is $\delta$-unbiased for $\delta = (1/2) e^{- 2\spty}$. Note that $w(n) \cdot X = \sum_{j < n} w(n)_j Z_j + w(n)_n$ and $v(n) \cdot X = \sum_{j < n} v(n)_j Z_j + v(n)_n$. Therefore, as $(Z_1,\ldots,Z_{n-1})$ is $\delta$-unbiased, by \lref{lm:l2toexact} and \eref{eq2:ell2error}, we get 

$$\max_{j < n} | v(n)_j - w(n)_j | \leq O(1) \exp(2 \spty) \cdot \sqrt{\gamma/\delta},$$
if $\gamma \leq c \delta \cdot \exp ( - 4 \spty) \leq c \exp(-5 \spty)$ for a sufficiently small $c$. Thus, if we set $\gamma = c' \exp(-5 \spty) \epsilon^2$ for a sufficiently small constant $c'$, then we get
$$\max_{j < n} |A_{nj} - \hat{A}_{nj}| = (1/2) \|v(n) - w(n)\|_\infty  \leq \epsilon.$$

By a similar argument for $i = 1,\ldots,n-1$ and taking a union bound, we get estimates $\hat{A}_{ij}$ for all $i \neq j$ so that with probability at least $1 - \rho$, 
$$\max_{i \neq j} |A_{ij} - \hat{A}_{ij}| \leq \epsilon.$$

Note that by \tref{th:pconceptmwu}, the number of samples needed to satisfy \eref{eq2:ell2error} is
$$O((\spty/\gamma)^2 \cdot (\log (n/\rho \gamma))) = O(\spty^2 \exp(10 \spty)/\epsilon^4) \cdot (\log (n/\rho \epsilon)).$$
This proves the theorem.
\end{proof}

The above theorem immediately implies an algorithm for recovering the dependency graph of an Ising model with nearly optimal sample complexity. 

\begin{corollary}\label{cor:ising}
Let $\calD(A,\theta)$ be an $n$-variable Ising model with width $\spty(A,\theta) \leq \spty$ and each non-zero entry of $A$ at least $\eta > 0$ in absolute value. There exists an algorithm that given $\spty$, $\eta,\rho  \in (0,1)$, and $N = O(\exp(O(\spty))/\eta^4) \cdot (\log (n/\rho \eta))$ independent samples $Z^1,\ldots,Z^N \leftarrow \calD(A,\theta)$ recovers the underlying dependency graph of $\calD(A,\theta)$ with probability at least $1-\rho$. The run-time of the algorithm is $O(n^2 N)$. Moreover, the algorithm can be run in an online manner.
\end{corollary}
\begin{proof}
The claim follows immediately from \tref{th:ising} by setting $\epsilon = \eta/2$ to compute $\hat{A}$ and taking the edges $E$ to be $\{\{i,j\}: |\hat{A}_{ij}| \geq \eta/2\}$. 
\end{proof}

It is instructive to compare the upper bounds from Corollary \ref{cor:ising} with known unconditional lower bounds on the sample complexity of learning Ising models with $n$ vertices due to Santhanam and Wainwright \cite{SW}.  They prove that, even if the weights of the underlying graph are known, any algorithm for learning the graph structure must use $\Omega(\frac{2^{\lambda/4} \cdot \log n}{\eta \cdot 2^{3 \eta}})$ samples. Hence, the sample complexity of our algorithm is near the best-known information-theoretic lower bound.

\section{Recovering polynomials from $\ell_2$ minimization}\label{sec:recpoly}

In order to obtain results for learning general Markov Random Fields, we need to extend our learning results from previous sections to the case of sigmoids of low-degree polynomials.  In this section, we prove that for any polynomial $p:\R^n \to \R$, minimizing the $\ell_2$-loss with respect to a sigmoid under a $\delta$-unbiased distribution $\calD$ also implies closeness as a polynomial. That is, for two polynomials $p,q:\R^n \to \R$ if $\ex_{X \sim \calD}[(\sigma(p(X)) - \sigma(q(X)))^2]$ is sufficiently small, then $\|p-q\|_1 \ll 1$ (\lref{lm:l2ptoexact}) and that the coefficients of maximal monomials of $p$ can be inferred from $q$ (\lref{lm:errstrong}). These results will allow us to recover the structure and parameters of MRFs when combined with Sparsitron. 

The exact statements and arguments here are similar in spirit to \lref{lm:l2toexact} and its proof but are more subtle. 
 To start with, we need the following property of $\delta$-unbiased distributions which says that low-degree polynomials are not too small with non-trivial probability (aka \emph{anti-concentration}) under $\delta$-unbiased distributions.
\begin{lemma}\label{lm:polylb}
There is a constant $c > 0$ such that the following holds. Let $D$ be a $\delta$-unbiased distribution on $\dpm^n$. Then, for any multilinear polynomial $s:\R^n \to \R$, and any maximal monomial $I \neq \emptyset \subseteq [n]$ in $s$, 
$$\pr_{X \sim \cal{D}}[ \abs{s(X)} \geq \abs{\widehat{s}(I)}] \geq \delta^{|I|}.$$
\end{lemma}
\begin{proof}
We prove the claim by induction on $|I|$. For an $i \in [n]$, let $x^{i,+} \in \dpm^n$ (respectively $x^{i,-}$) denote the vector obtained from $x$ by setting $x_i = 1$ (respectively $x_i = -1)$. Note that $x^{i,+}, x^{i,-}$ only depend on $x_{-i}$. Let $X \sim \cal{D}$.

Suppose $I = \{i\}$  so that $s(x) = \widehat{s}(\{i\}) x_i + s'(x_{-i})$ for some polynomial $s'$ that only depends on $x_{-i}$. Note that $\max(|s(x^{i,+})|, |s(x^{i,-})|) \geq |\widehat{s}(\{i\})|$. Therefore, for any fixing of $X_{-i}$, as $X$ is $\delta$-unbiased, 
$$\pr_{X_i | X_{-i}}[ |s(X)| \geq |\widehat{s}(\{i\})|] \geq \delta.$$

Now, suppose $|I| = \ell \geq 2$ and that the claim is true for all polynomials and all monomials of size at most $\ell-1$. Let $i \in I$. Then, $s(x) = x_i \cdot \partial_i(s(x_{-i})) + s'(x_{-i})$ for some polynomial $s'$ that only depends on $x_{-i}$. Thus, $\max(|s(x^{i,+})|, |s(x^{i,-})|) \geq |\partial_i s(x_{-i})|$. Therefore, for any fixing of $X_{-i}$, as $X$ is $\delta$-unbiased, 
$$\pr_{X_i | X_{-i}}[ |s(X)| \geq |\partial_i s(X_{-i})|] \geq \delta.$$
Now, let $J = I \setminus \{i\}$ and observe that $J$ is a maximal monomial in $r(x_{-i}) \equiv \partial_i s(x_{-i})$ with $\widehat{r}(J) = \widehat{s}(I)$. Therefore, by the induction hypothesis, 
$$\pr_{X_{-i}}[ |\partial_i s(X_{-i})| \geq \abs{\widehat{s}(I)}] \geq \delta^{\ell-1}.$$
Combining the last two inequalities, we get that $\pr[|s(X)| \geq \widehat{s}(I)] \geq \delta^\ell$. The claim now follows by induction.
\end{proof}

The next lemma shows that for unbiased distributions $\cal{D}$, and two low-degree polynomials $p,q:\R^n \to \R$, if $\ex_{X \sim \cal{D}}[ (\sigma(p(x)) - \sigma(q(x))^2]$ is small, then one can infer the coefficients of the maximal monomials of $p$ from $q$\footnote{Note that under the hypothesis of the lemma, the coefficients of $p$ and $q$ can nevertheless be far.}. 
\begin{lemma}\label{lm:errstrong}
Let $\calD$ be a $\delta$-unbiased distribution on $\dpm^n$. Let $p,q$ be two multilinear polynomials $p,q:\R^n \to \R$ such that $\ex_{X \sim \cal{D}}[ (\sigma(p(x)) - \sigma(q(x))^2] \leq \epsilon$. Then, for every maximal monomial $I \subseteq [n]$ of $p$, and any $\rho > 0$, 
$$\pr_{X \sim \cal{D}}\sbkets{\abs{\widehat{p}(I) - \partial_I q(X)} > \rho} \leq  \frac{  e^{2 \|p\|_1 + 6} \epsilon}{\rho^2 \delta^{|I|}}.$$
\end{lemma}
\begin{proof}
Let $X \sim \cal{D}$ and fix a maximal monomial $I\subseteq [n]$ in $p$. Now, for any $x \in \dpm^n$, by \clref{clm:sigmoidlb}, 
$$\abs{\sigma(p(x)) - \sigma(q(x))} \geq e^{-\|p\|_1 - 3} \cdot \min\bkets{1, \abs{p(x) - q(x)}}.$$

Therefore, 
$$\ex\sbkets{\min\bkets{1,\abs{p(X) - q(X)}^2}} \leq e^{2 \|p\|_1 + 6} \epsilon.$$ 

Hence, for every $\rho \in (0,1)$, 
$$\pr_X\sbkets{ \abs{p(X) - q(X)} > \rho} \leq e^{2 \|p\|_1 + 6} \epsilon/\rho^2.$$

Now consider a fixing of all variables not in $I$ to $z \in \dpm^{[n] \setminus I}$ and let $r_z(x_I)$ be the polynomial obtained by the resulting fixing. Now, 
$$\pr_X\sbkets{ \abs{p(X) - q(X)} > \rho} = \sum_{z \in \dpm^{[n] \setminus I}} \pr[X_{[n] \setminus I} = z] \cdot \pr[\abs{r_z(X_I)} > \rho \mid X_{[n] \setminus I} = z].$$
Further, $\widehat{r}(I) = \widehat{p}(I) - \partial_I q(z)$ as $I$ is maximal in $p$. 

Conditioned on the event that $|\widehat{r}(I)| > \rho$, for a random choice of $X_{[n] \setminus I}$, we have from \lref{lm:polylb} that $\pr_{X_I}\sbkets{|r_z(X_I)|> \rho} \geq \delta^{|I|}$.  Thus we have

$$ \pr_X\sbkets{ \abs{p(X) - q(X)} > \rho} \geq  \delta^{|I|} \cdot \pr_{X_{[n] \setminus I}}\sbkets{  \abs{\widehat{p}(I) - \partial_I q(X_{[n] \setminus I})} > \rho}. $$



Combining the above equations we get that
$$\pr_{X}\sbkets{  \abs{\widehat{p}(I) - \partial_I q(X)} > \rho} \leq \frac{ e^{2 \|p\|_1 + 6} \epsilon}{\rho^2 \delta^{|I|}}.$$
\end{proof}

The next claim shows that under the assumptions of \lref{lm:errstrong}, the highest degree monomials of $p,q$ are close to each other. 

\begin{lemma}\label{lm:l2ptoexactm}
Let $\calD$ be a $\delta$-unbiased distribution on $\dpm^n$. Let $p,q$ be two multilinear polynomials $p,q:\R^n \to \R$ such that $\ex_{X \sim \cal{D}}[ (\sigma(p(x)) - \sigma(q(x))^2] \leq \epsilon$ where $\epsilon <   e^{-2\|p\|_1 - 6} \delta^{|I|}$. Then, for every maximal monomial $I \subseteq [n]$ of $(p-q)$,
$$\abs{\widehat{p}(I) - \widehat{q}(I)} \leq   e^{\|p\|_1+3} \cdot \sqrt{\epsilon/\delta^{|I|}}.$$
\end{lemma}
\begin{proof}
Fix a maximal monomial $I\subseteq [n]$ in $(p-q)$. Now, for any $X$, by \clref{clm:sigmoidlb}, 
$$\abs{\sigma(p(X)) - \sigma(q(X))} \geq e^{-\|p\|_1 - 3} \cdot \min\bkets{1, \abs{p(X) - q(X)}}.$$
On the other hand, as $X$ is $\delta$-unbiased, by \lref{lm:polylb}, with probability at least $\delta^{|I|}$, $\abs{p(X) - q(X)} \geq \abs{\widehat{p}(I) - \widehat{q}(I)}$. Therefore, 
$$\epsilon \geq \ex_X\sbkets{\bkets{\sigma(p(X)) - \sigma(q(X))}^2} \geq e^{-2 \|p\|_1 - 6} \cdot \delta^{|I|} \cdot \min \bkets{1, \abs{\widehat{p}(I) - \widehat{q}(I)}^2}.$$

As $\epsilon <  e^{-2\|p\|_1 - 6} \delta^{|I|}$, the above inequality can only hold if $\abs{\widehat{p}(I) - \widehat{q}(I)} < 1$ so that
$$\abs{\widehat{p}(I) - \widehat{q}(I)} < e^{\|p\|_1 + 3} \sqrt{\epsilon/\delta^{|I|}}.$$
The claim follows. 
\end{proof}


We next show that if $\ex_{X \sim \cal{D}}[ (\sigma(p(x)) - \sigma(q(x))^2] \ll n^{-t}$ is sufficiently small, then $\|p-q\|_1 \ll 1$. 
\begin{lemma}\label{lm:l2ptoexact}
Let $D$ be a $\delta$-unbiased distribution on $\dpm^n$. Let $p,q$ be two multilinear polynomials $p,q:\R^n \to \R$ of degree $t$ such that $\ex_{X \sim \cal{D}}[ (\sigma(p(x)) - \sigma(q(x))^2] \leq \epsilon$ where $\epsilon < e^{-2 \|p\|_1 - 6} \delta^{t}$. Then, 
$$\|p-q\|_1 = O(1) \cdot  (2t)^t e^{\|p\|_1} \cdot \sqrt{\epsilon/\delta^t} \cdot \binom{n}{t}.$$
\end{lemma}
\begin{proof}
For a polynomial $s:\R^n \to \R$ of degree at most $t$, and $\ell \leq t$, let $s_{\leq \ell}$ denote the polynomial obtained from $s$ by only taking monomials of degree at most $\ell$ and let $s_{=\ell}$ denote the polynomial obtained from $s$ by only taking monomials of degree exactly $\ell$. 

For brevity, let $r = p-q$, and for $\ell \leq t$, let $\rho_\ell = \|r_{=\ell}\|_1 = \|p_{=\ell}  - q_{=\ell}\|_1$. We will inductively bound $\rho_t, \rho_{t-1},\ldots,\rho_1$. 

From \lref{lm:l2ptoexactm} applied to the polynomials $p,q$, we immediately get that 
\begin{equation}\label{eq:recbase}
\rho_t = \|r_{=t}\|_1 \leq e^{\|p\|_1+3} \cdot \sqrt{\epsilon/\delta^{t}} \cdot \binom{n}{t} \equiv \epsilon_0.
\end{equation}

Now consider $I \subseteq [n]$ with $|I| = \ell$. Then, by an averaging argument, there is some fixing of the variables not in $X_I$ so that for the polynomials $p_I, q_I$ obtained by this fixing, and for the resulting distribution $\cal{D}_I$ on $\dpm^I$, 
$$\ex_{Y \sim \cal{D}_I}[ (\sigma(p_I(Y)) - \sigma(q_I(Y)))^2] \leq \epsilon.$$

Note that $\cal{D}_I$ is also $\delta$-unbiased. Therefore, by \lref{lm:l2ptoexactm} applied to the polynomials $p,q$, letting $r_I = p_I - q_I$, we get that 
$$\abs{\widehat{r}_I(I)} = \abs{\widehat{p_I}(I) - \widehat{q}_I(I)} \leq e^{\|p\|_1+3} \cdot \sqrt{\epsilon/\delta^{|I|}}.$$
We next relate the coefficients of $r_I$ to that of $r$. As the polynomial $r_I$ is obtained from $r$ by fixing the variables not in $I$ to some values in $\dpm$,
$$\abs{\widehat{r}_I(I)} \geq \abs{\widehat{r}(I)} - \sum_{J: J \supset I} \abs{\widehat{r}(J)}.$$
Combining the above two inequalities, we get that
$$\abs{\widehat{r}(I)} \leq e^{\|p\|_1+3} \cdot \sqrt{\epsilon/\delta^{\ell}} + \sum_{J \supset I} \abs{\widehat{r}(J)}.$$
Summing the above equation over all $I$ of size exactly $\ell$, we get
\begin{align*}
\|r_{=\ell}\|_1 &= \sum_{I: |I| = \ell} \abs{\widehat{r}(I)} \leq e^{\|p\|_1+3} \cdot \sqrt{\epsilon/\delta^{\ell}} \cdot \binom{n}{\ell} + \sum_{I: |I| = \ell} \bkets{\sum_{J \supset I} \abs{\widehat{r}(J)}}\\
&\leq \epsilon_0 + \sum_{I: |I| = \ell} \bkets{\sum_{J \supset I} \abs{\widehat{r}(J)}}\\
&= \epsilon_0 + \sum_{j = \ell+1}^{t} \binom{j}{\ell} \cdot \bkets{\sum_{J: |J| = j} \abs{\widehat{r}(J)}}= \epsilon_0 + \sum_{j=\ell+1}^t \binom{j}{\ell} \|r_{=j}\|_1.
\end{align*}
Therefore, we get the recurrence,
\begin{equation}\label{eq:erreccurence}
\rho_\ell \leq \epsilon_0 + \sum_{j=\ell+1}^t \binom{j}{\ell} \rho_j.
\end{equation}
We can solve the above recurrence by induction on $\ell$. Specifically, we claim that the above implies $\rho_j \leq (2t)^{t-j} \cdot \epsilon_0$. For $j = t$, the claim follows from \eref{eq:recbase}. Now, suppose the inequality holds for all $j > \ell$. Then, by \eref{eq:erreccurence}, as $\binom{j}{\ell}  \leq j^{j-\ell}$, 
\begin{align*}
\rho_\ell &\leq \epsilon_0 + \sum_{j=\ell+1}^t j^{j-\ell} (2t)^{t-j} \epsilon_0 \leq \epsilon_0 + \sum_{j=\ell+1}^t t^{j-\ell} (2t)^{t-j} \epsilon_0\\
&\leq t^{t-\ell} \cdot \epsilon_0 \cdot \bkets{1 + \sum_{j=\ell+1}^t 2^{t-j}} = t^{t-\ell} \cdot \epsilon_0 \cdot 2^{t-\ell}.
\end{align*}

Therefore,
$$\|r\|_1 = \sum_{\ell=0}^t \|r_{=\ell}\|_1 \leq \sum_{\ell=0}^t (2t)^{t-\ell} \epsilon_0 \leq \epsilon_0 \cdot 2^{t+1} t^t.$$
The lemma now follows by plugging in the value of $\epsilon_0$.
\end{proof}

\section{Learning Markov Random Fields} \label{sec:mrf}
We now describe how to apply the Sparsitron algorithm to recover the structure as well as parameters of binary $t$-wise MRFs. 

We will use the characterization of MRFs via the Hammersley-Clifford theorem. Given a graph $G = (V,E)$ on $n$ vertices, let $C_t(G)$ denote all cliques of size at most $t$ in $G$. A binary $t$-wise MRF with dependency graph $G$ is a distribution $\calD$ on $\dpm^n$ where the probability density function of $\calD$ can be written as

$$\pr_{Z \sim \calD} [Z= x] \propto \exp\bkets{\sum_{I \in \cal{S}} \psi_I(x)},$$
where $\cal{S} \subseteq C_t(G)$ and each $\psi_I:\R^n \to \R$ is a function that depends only on the variables in $I$. Note that if $t=2$, this corresponds exactly to the Ising model. We call $\psi(x) = \sum_{I \in \cal{S}} \psi_I(x)$ the \emph{factorization polynomial} of the MRF and $G$ the \emph{dependency graph} of the MRF.

Note that the factorization polynomial is a polynomial of degree at most $t$. However, different graphs and factorizations (i.e., functions $\{\psi_I\}$) could potentially lead to the same polynomial. To get around this we enforce the following non-degeneracy condition:
\begin{definition}
For a $t$-wise MRF $\calD$ on $\dpm^n$ we say an associated dependency graph $G$ and factorization 
$$\pr_{Z \sim \calD}[Z= x] \propto \exp\bkets{\sum_{I \in \cal{S}} \psi_I(x)},$$
for $\cal{S} \subseteq C_t(G)$ is $\eta$-identifiable if for every \emph{maximal monomial} $J$ in $\psi(x) = \sum_{I \in \cal{S}} \psi_I(x)$, $\abs{\widehat{\psi}(J)} \geq \eta$ and every edge in $G$ is covered by a non-zero monomial of $\psi$.

\end{definition}

 We now state our main theorems for learning MRFs. Our first result is about \emph{structure learning}, i.e., recovering the underlying dependency graph of a MRF. Roughly speaking, using $N = 2^{O(\spty t)} \log(n/\eta)/\eta^4$ samples we can recover the underlying dependency graph of a $\eta$-identifiable MRF where $\spty$ is the maximum $\ell_1$-norm of the derivatives of the factorization polynomial. The run-time of the algorithm is $O(M \cdot n^t)$. Note that $\max_i \|\partial_i \psi\|_1$ is analogous to the notion of width for Ising models (as in \cref{cor:ising}). Thus, exponential dependence on it is necessary as in the Ising model and our sample complexity is in fact nearly optimal in all parameters. 
 
 \begin{theorem}\label{th:identifiable}
Let $\calD$ be a $t$-wise MRF on $\dpm^n$ with underlying dependency graph $G$ and factorization polynomial $p(x) = \sum_{I \in C_t(G)} p_I(x)$ with $\max_i \|\partial_i p\|_1 \leq \spty$. Suppose that $\calD$ is $\eta$-identifiable. Then, there exists an algorithm that given $\spty$, $\eta,\rho \in (0,1/2)$, and 
$$N = \frac{e^{O(t)} e^{O( \spty t)}}{\eta^4}  \cdot (\log (n/\rho \eta))$$ independent samples from $\calD$, recovers the underlying dependency graph $G$ with probability at least $1-\rho$. The run-time of the algorithm is $O(N \cdot n^t)$. Moreover, the algorithm can be run in an online manner. 
\end{theorem}

Along with learning the dependency graph, given more samples, we can also approximately learn the parameters of the MRF: i.e., compute a $t$-wise MRF whose distribution is close as a pointwise-approximation to the original probability density function.

\begin{theorem}\label{th:parametersmrf}
Let $\calD$ be a $t$-wise MRF on $\dpm^n$ with underlying dependency graph $G$ and factorization polynomial $\psi(x) = \sum_{I \in C_t(G)} \psi_I(x)$ with $\max_i \|\partial_i \psi\|_1 \leq \spty$. There exists an algorithm that given $\spty$, and $\epsilon,\rho \in (0,1/2)$, and 
$$N = \frac{(2t)^{O(t)} e^{O( \spty t)}}{\epsilon^4} \cdot n^{4t} \cdot (\log (n/\rho \epsilon))$$ independent samples $Z^1,\ldots,Z^N \leftarrow \calD$ produces a $t$-wise MRF $\calD'$ with dependency graph $H$ and a factorization polynomial $\phi(x) = \sum_{I \in C_t(H)} \phi_I(x)$ such that with probability at least $1-\rho$:
$$\forall x,\; \pr_{Z \sim \calD}[ Z = x] = (1\pm \epsilon) \pr_{Z \sim \calD'}[Z=x].$$

The algorithm runs in time $O(N n^t)$ and can be run in an online manner.
\end{theorem}

 We in fact show how to recover the parameters of a \emph{log-polynomial} density defined as follows:
\begin{definition}
A distribution $\calD$ on $\dpm^n$ is said to be a log-polynomial distribution of degree $t$ if for some multilinear polynomial $p:\R^n \to \R$ of degree $t$,
$$\pr_{X \sim \calD}[X = x] \propto \exp( p(x)).$$
\end{definition}
 
\begin{theorem} \label{th:mainpoly}
Let $\calD$ be a log-polynomial distribution of degree at most $t$ on $\dpm^n$ with the associated polynomial $p:\R^n \to \R$ such that $\max_i \|\partial_i p\|_1 \leq \spty$. There exists an algorithm that given $\spty$, and $\epsilon,\rho \in (0,1)$ and 
$$N = \frac{(2t)^{O(t)} \cdot e^{O(\spty t)}}{\epsilon^4} \cdot (\log(n/\rho \epsilon)),$$
independent samples $Z^1,\ldots,Z^N \leftarrow \calD$, finds a multilinear polynomial $q:\R^n \to \R$ such that  with probability at least $1-\rho$
$$\|p-q\|_1 \leq \epsilon \cdot \binom{n}{t}.$$

Moreover, we can also find coefficients $(\hat{s}(I): I \subseteq [n], |I| \leq t)$ such that with probability at least $1-\rho$, for every maximal monomial $I$ of $p$, we have $\abs{\widehat{p}(I) -  \widehat{s}(I)} < \epsilon$.  The run-time of the algorithm is $O(N \cdot n^t)$ and the algorithm can be run in an online manner.
\end{theorem}

\subsection{Learning the structure of MRFs}
The following elementary properties of MRFs play a critical role in our analysis. 
\begin{lemma}\label{lm:mrfprops}
Let $\calD$ be a $t$-wise MRF on $\dpm^n$ with underlying dependency graph $G$ and factorization polynomial $p(x) = \sum_{I \in C_t(G)} p_I(x)$ with $\max_i \|\partial_i p\|_1 \leq \spty$. Then, the following hold for $Z \leftarrow \calD$:
\begin{itemize}
\item For any $i$, and a partial assignment $x \in \dpm^{[n] \setminus \{i\}}$, $\pr[Z_i = -1 | Z_{-i} = x] = \sigma(- 2 \partial_i p(x))$.
\item $\calD$ is $(e^{-2\spty}/2)$-unbiased.
\end{itemize}
\end{lemma}
\begin{proof}
For any $x \in \dpm^{[n] \setminus \{i\}}$, 
$$\frac{\pr[Z_i = 1 | Z_{-i} = x]}{\pr[Z_i = -1 | Z_{-i} = x]} = \exp(2 \partial_i p(x)).$$
Thus, 
$$\pr[Z_i = -1 | Z_{-i} = x] = \sigma(- 2 \partial_i p(x)).$$

Next, for each $i$, and any partial assignment $x$ to $Z_{-i}$,
\begin{multline*}
\min\bkets{\pr[Z_i = -1 | Z_{-i} = x], \pr[Z_i = 1 | Z_{-i} = x]} =\\ \min\bkets{\sigma(-2 \partial_i p(x)), 1- \sigma(-2\partial_ip (x))} \geq (1/2) e^{-2\|\partial_i p\|_1} \geq (1/2) e^{-2 \spty}.
\end{multline*}
\end{proof}

We also need the following elementary fact about median:
\begin{claim}\label{clm:median}
Let $X$ be a real-valued random variable such that for some $\alpha, \gamma \in \R$, $\pr[| X - \alpha| > \gamma] < 1/4$. Then, for $K$ independent copies of $X$, $X_1, X_2,\ldots,X_K$, 
$$\pr[ \abs{\textsc{Median}(X_1,\ldots,X_K)- \alpha} > \gamma] \leq 2 \exp(-\Omega( K)).$$
\end{claim}

\begin{proof}[Proof of \tref{th:identifiable}]
We will show how to recover neighbors of the vertex $n$ (for ease of notation). By repeating the argument for all $i \in [n]$, we will get the graph $G$.

The starting point for our algorithm is \lref{lm:mrfprops} that allow us to use Sparsitron algorithm via \emph{feature expansion} and the properties of $\delta$-unbiased distributions developed in \sref{sec:recpoly}. 

Concretely, let $p' = -2 \partial_n p$ and $\bm{p'} = (\widehat{p'}(I): I \subseteq [n-1], |I| \leq t-1)$. Similarly, for $x \in \dpm^{n-1}$, let $\bm{v}(x) = (\prod_{i \in I} x_i: I \subseteq [n-1], |I| \leq t-1)$. Let $Z \sim \calD$ and $X$ be the distribution of $\bm{v}(Z_{-n})$ and let $Y = (1 - Z_n)/2$. Then, by \lref{lm:mrfprops}, we have
$$\ex[Y | X ] = \sigma( \bm{p'} \cdot X).$$

Let $\delta = e^{-2 \spty}/2$, and let $\epsilon \in (0,1)$, $K \geq 1$ be parameters to be chosen later. Our algorithm is shown in Figure \ref{recover}. The intuition is as follows: We first apply Sparsitron to recover a polynomial $q$ that approximates $\partial_n p$ in the sense that 
$$\ex_Z[ \bkets{\sigma(-2\partial_n p(Z)) - \sigma(-2q(Z))}^2] < \epsilon.$$ 

However, the above does not guarantee that the coefficients of $q$ are close to those of $\partial_n p$. To overcome this, we exploit \lref{lm:errstrong} that guarantees that for any maximal monomial $I$ in $\partial_n p$, $\partial_I q(Z)$ is close to $\widehat{\partial_n p}(I)$ with high probability for $Z \sim \calD$; concretely, in steps (4), (5), (6), we draw fresh samples from $\calD$ and use the median evaluation of $\partial_I q(\;\;)$ as our estimate for $\widehat{\partial_n p}(I)$. 

\begin{algorithm}
\caption{\textsc{MRF Recovery} \label{recover}}
\begin{algorithmic}[1]
\State Initialize $H = \emptyset$ to be the empty graph. 
\State Apply the Sparsitron algorithm as in \tref{th:pconceptmwu} to compute a vector $\bm{q}$ such that with probability at least $1-\rho/2n^2$, 
$$\ex[(\sigma(\bm{p'} \cdot X) - \sigma( \bm{q} \cdot X))^2] \leq \epsilon.$$
\State Define a polynomial $q : \R^{n-1} \to \R$ by setting $\widehat{q}(I) = (-1/2) \bm{q}_I$ for all $I \subseteq [n-1]$. 
\State Let $Z^1,\ldots,Z^K$ be additional independent samples from $\calD$.
\For{each $I \subseteq [n-1]$, $|I| \leq t-1$}
\State If $\abs{\textsc{median}\bkets{\partial_I q(Z^1), \ldots, \partial_I q(Z^K)}} > \eta/2$, then add the complete graph on $\{n\} \cup I$ to $H$. 
\EndFor
\end{algorithmic}
\end{algorithm}

We next argue that for a suitable choice of $\epsilon, K$, with probability at least $1-\rho/n$, the graph $H$ contains all edges of $G$ adjacent to vertex $n$.

Observe that by our definitions of $\bm{p'}, \bm{q}, X$ 
$$\ex_{Z}\sbkets{\bkets{\sigma(-2 \partial_n p(Z)) - \sigma(-2 q(Z)}^2} = \ex[(\sigma(\bm{p'} \cdot X) - \sigma( \bm{q} \cdot X))^2] \leq \epsilon.$$

(Here, we abuse notation and write $q(Z) = q(Z_1,\ldots,Z_{n-1})$ as the latter does not depend on $Z_n$.)

Further, as $Z$ is $\delta$-unbiased by \lref{lm:mrfprops}, by \lref{lm:errstrong} for any maximal monomial $I \subseteq [n-1]$ of $\partial_n p$, we have
$$\pr\sbkets{ \abs{\widehat{\partial_n p}(I) - \partial_I q(Z)} > \eta/4} < \frac{  16 e^{2 \|p\|_1 + 6} \epsilon}{\eta^2 \delta^{|I|}}.$$

Let $\epsilon = e^{-2 \spty - 6} \eta^2 \delta^{t}/64$ so that 

$$\pr\sbkets{ \abs{\widehat{\partial_n p}(I) - \partial_I q(Z)} > \eta/4} < 1/4.$$
Therefore, by \clref{clm:median}, 
$$\pr\sbkets{\abs{\textsc{Median}(\partial_I q(Z^1), \partial_I q(Z^2),\ldots,\partial_I q(Z^K)) - \widehat{\partial_n p}(I)} >  \eta/4} < 2 \exp(-\Omega(K)).$$

Taking $K = C \log(n^t/\rho)$ for a sufficiently big constant $C$, we get that with probability at least $1-\rho/n$, for all maximal monomials $I$ of $\partial_n p$, $\abs{\textsc{Median}(\partial_I q(Z^1), \partial_I q(Z^2),\ldots,\partial_I q(Z^K)) - \widehat{\partial_n p}(I)} < \eta/4$. 

Now, whenever the above happens, as the coefficients of maximal monomials of $p$ are at least $\eta$ in magnitude (by $\eta$-identifiability), our algorithm will add the complete graph on the variables of all maximal monomials of $p$ involving vertex $n$ to $H$. 

 Thus, the algorithm recognizes the neighbors of vertex $n$ exactly with probability at least $1-\rho/n$. Repeating the argument for each vertex $i \in [n]$ and taking a union bound over all vertices gives us the recovery guarantee of the theorem. It remains to bound the sample-complexity.

Note that $\|\bm{p'}\|_1 = 2\|\partial_n p\|_1 \leq 2\spty$. Therefore, by \tref{th:pconceptmwu}, the number of samples needed for the call to Sparsitron in Step (2) of Algorithm \ref{recover} is 
$$O(\spty^2 \cdot \ln(n^t/\rho \epsilon)/\epsilon^2) = e^{O(t)} \cdot e^{O(\spty \cdot t)} \cdot \ln(n/\rho \eta) \cdot (1/\eta^4).$$

As $K = Ct \ln(n/\rho)$, the above bound dominates the number of samples proving the theorem.
\end{proof}

\subsection{Learning log-polynomial densities and parameters of MRFs}
We first observe that \tref{th:mainpoly} implies \tref{th:parametersmrf}

\begin{proof}[Proof of \tref{th:parametersmrf}] 
We apply \tref{th:mainpoly} with error $\epsilon' = \epsilon n^{-t}$ to samples from $\calD$ to obtain a polynomial $\phi:\R^n \to \R$ such that $\|\psi - \phi\|_1 \leq \epsilon$. We build a new graph $H$ as follows: For each monomial $I \subseteq [n]$ with $\widehat{\phi}(I) \neq 0$, add all the edges in $I$ to $H$. Let $\calD'$ denote the $t$-wise MRF with dependency graph $H$ and factorization polynomial $\phi$. Since, $\|\psi - \phi\|_1 \leq \epsilon$, it follows that for all $x$, $|\psi(x) - \phi(x)| < \epsilon$. Therefore, for all $x$,
$$\exp(\psi(x)) = \exp(\phi(x) \pm \epsilon) = (1 \pm 2 \epsilon) \exp(\phi(x)).$$
The theorem now follows. 
\end{proof}

We next prove \tref{th:mainpoly}. The proof is similar to that of \tref{th:ising} and \tref{th:identifiable}.

\begin{proof}[Proof of \tref{th:mainpoly}]
For each $i$, we will show how to recover a polynomial $q_i$ such that $\|\partial_ip - q_i \|_1 < \epsilon \cdot \binom{n}{t-1}$. We can then combine these polynomials to obtain a polynomial $q$. One way to do so is as follows: For each $I \subseteq [n]$, let $i = \arg\min(I)$, and define $\widehat{q}(I) = \widehat{q_i}(I \setminus \{i\})$. Then,
\begin{align*}
\|p-q\|_1 &= \sum_{I} \abs{\widehat{p}(I) - \widehat{q}(I)}= \sum_{i=1}^n \sum_{I: \arg\min(I) = i}  \abs{\widehat{p}(I) - \widehat{q}(I)}\\
&\leq \sum_{i=1}^n \|\partial_i p - q_i\| \leq \epsilon \cdot n \cdot \binom{n}{t-1}.
\end{align*}

Here we show how to find a polynomial $q_n$ such that with probability at least $1-\rho/n$,
\begin{equation}\label{eq:poly2}
\|\partial_n p - q_n\|_1 < \epsilon \cdot \binom{n}{t-1}.
\end{equation}
The other cases can be handled similarly and the theorem then follows from the above argument. 

\ignore{
The starting point for our algorithm is the following observation. Let $Z \sim \calD$. Then, for any $x \in \dpm^{n-1}$, 
$$\frac{\pr[Z_n = 1 | Z_{-n} = x]}{\pr[Z_n = -1 | Z_{-n} = x]} = \exp(2 \partial_n p(x)).$$
Thus, 
$$\pr[Z_n = -1 | Z_{-n} = x] = \sigma(- 2 \partial_n p(x)).$$}

As in \tref{th:identifiable}, we exploit \lref{lm:mrfprops} to employ our Sparsitron algorithm for learning GLMs via \emph{feature expansion}. Concretely, let $p' = -2 \partial_n p$ and $\bm{p'} = (\widehat{p'}(I): I \subseteq [n-1], |I| \leq t-1)$. Similarly, for $x \in \dpm^{n-1}$, let $\bm{v}(x) = (\prod_{i \in I} x_i: I \subseteq [n-1], |I| \leq t-1)$. Let $Z \sim \calD$ and $X$ be the distribution of $\bm{v}(x)$ and let $Y = (1 - Z_n)/2$. Then, from the above arguments, we have
$$\ex[Y | X ] = \sigma( \bm{p'} \cdot X).$$

Note that $\|\bm{p'}\|_1 = 2 \|\partial_n p\|_1 \leq 2 \spty$. Let $\gamma \in (0,1)$ be a parameter to be chosen later. We now apply the Sparsitron algorithm as in \tref{th:pconceptmwu} to compute a vector $\bm{q'} \in \R^n$ such that with probability at least $1-\rho/n$, 
$$\ex[(\sigma(\bm{p'} \cdot X) - \sigma( \bm{q'} \cdot X))^2] \leq \gamma.$$

We define polynomial $q_n$ by setting $\widehat{q_n}(I) = (-1/2) \cdot \bm{q'}_I$ for all $I \subseteq [n-1]$. Then, the above implies that
\begin{equation}\label{eq:errorpm}
\ex_{Z}\sbkets{\bkets{\sigma(-2 \partial_n p(Z)) - \sigma(-2 q_n(Z)}^2} \leq \gamma.
\end{equation}
\ignore{
Finally, just as in the Ising model, for each $i$, and any partial assignment $x$ to $Z_{-i}$,
\begin{multline*}
\min\bkets{\pr[Z_i = -1 | Z_{-i} = x], \pr[Z_i = 1 | Z_{-i} = x]} =\\ \min\bkets{\sigma(-2 \partial_i p(x)), 1- \sigma(-2\partial_ip (x))} \geq (1/2) e^{-2\|\partial_i p\|_1} \geq (1/2) e^{-2 \spty}.
\end{multline*}}

Now, an argument similar to that of \lref{lm:mrfprops} shows that $Z$ is $\delta$-unbiased for $\delta = e^{-2\spty}/2$. Therefore, by \eref{eq:errorpm}, and \lref{lm:l2ptoexact}, for $\gamma < c \exp(-4 \spty ) \cdot \delta^{-t}$ for a sufficiently small constant $c$, we get 
$$\|\partial_n p-q_n\|_1 \leq O(1) (2t)^t \cdot e^{2\spty} \cdot \sqrt{\gamma/\delta^t} \cdot \binom{n}{t-1} \leq \epsilon \cdot \binom{n}{t-1},$$
where $\gamma = \epsilon^2 \cdot \exp(- C \spty t)/C (2t)^{2t}$ for a sufficiently large constant $C > 0$. Note that by \tref{th:pconceptmwu}, the number of samples needed to satisfy \eref{eq:errorpm} is
$$O((\spty/\gamma)^2 \cdot (\log (n/\rho \gamma)) = \frac{(2t)^{O(t)} \cdot e^{O(\spty t)}}{\epsilon^4} \cdot (\log (n/\rho \epsilon)).$$
This proves \eref{eq:poly2} and hence the main part of the theorem. The moreover part of the statement follows from an argument nearly identical to that of \tref{th:identifiable} and is omitted here. 
\end{proof}


\section{Extension of the Ising Model to general alphabet}
Here we extend our results for the Ising model from \sref{sec:ising} to work over general alphabet. 
\begin{definition}
Let $\cal{W} = (W_{ij} \in \R^{k \times k}: i \neq j \in [n])$ be a collection of matrices and $\theta \in \R^{[n] \times [k]}$. Then, the non-binary Ising model $\cal{D} \equiv \cal{D}(\cal{W},\theta)$ is the distribution on $[k]^n$ where
$$\pr_{X \sim \calD} \sbkets{X = x} \propto \exp \bkets{\sum_{i \neq j} W_{ij}(x_i,x_j) + \sum_{i=1}^n \theta_i(x_i)}.$$
The \emph{dependency graph} $G$ of $\cal{D}$ is the graph formed by pairs $\{i, j\}$ such that $W_{ij} \neq 0$. The \emph{width} of $\cal{D}$ is $\lambda(\calD) = \max_{i,a} \bkets{\sum_{j\neq i} \max_{b \in [k]} \abs{W_{ij}(a,b)} + \theta_i(a)}$. 

\end{definition}
It is easy to see that for $k=2$, the above description corresponds exactly to our discussion of Ising models in \sref{sec:ising}. Our goal here is to learn the structure and parameters of a general alphabet model as above given samples from it. We will show that arguments from \sref{sec:ising} can be extended to case of general alphabet as well leading to algorithms to learn the structure and parameters of distributions $\calD(\cal{W},\theta)$ with nearly optimal sample complexity. First, we address some necessary non-degeneracy conditions. 
\ignore{
By definition, for the non-binary Ising model
${\cal D(W,\theta)}$ we have 

where we view each $W_{ij}$ as a $k \times k$ matrix.  We define the {\em width} of the model, $\lambda(\calD)$ to be $\max_{i,a,b}
\bkets{\sum_{j,c} |W_{ij}(a,c)|} + |\theta_{i}(b)|$. }

Note that as described above, different parameter settings can lead to the same probability density function. To get around this, throughout this section we assume without loss of generality that for every $\{i,j\}$, the rows and columns of the matrices $W_{ij} \in \R^{k \times k}$ are centered, i.e., sum to zero. We can do so because of the following elementary claim: 

\begin{fact} \label{factnb}
Given $\cal{D}(\cal{W},\theta)$ as above, we may assume without loss of generality that for every $i,j$, the rows and columns of the matrix $W_{ij}$ are centered, i.e., have mean zero.
\end{fact}
\begin{proof}

For simplicity fix a particular $x = x_{1},\ldots,x_{n}$.  The above pdf
simplifies to 

$$\pr\sbkets{X = x} \propto \exp \bkets{\sum_{i \neq j} W_{ij}(x_{i},x_{j})  + \sum_i \theta_{i}(x_{i})}.$$

Now define $W_{ij}'(x_{i},x_{j})$ equal to $W_{ij}(x_{i},x_{j}) -
(1/k)\sum_{a}W_{ij}(a,x_{j})$ and define $\theta_{j}'(x_{j})$ equal to
$\theta_{j}(x_{j}) + \sum_{i < j} (1/k) \sum_{a} W_{ij}(a,x_{j})$.
Then by inspection the column sums of $W_{ij}'$ equal zero and
$\sum_{ij} W_{ij}(x_i,x_j) + \sum_{j} \theta_j(x_j) = \sum_{ij}
W_{ij}'(x_i,x_j) + \sum_j \theta'(x_j)$. 
Repeating the argument for rows gives us the claim.
\end{proof}

Next, we make the following identifiability assumption analogous to our assumptions for the Ising model.
\begin{definition}
$\cal{D}(\cal{W},\theta)$ is $\eta$-identifiable if for every edge in the dependency graph of $\calD$, $\|W_{ij}\|_\infty = \max_{a,b}(\abs{W_{ij}(a,b)})\geq \eta$ (assuming the matrices $W_{ij}$ are centered).
\end{definition}

We prove the following extension of \tref{th:ising} to non-binary Ising models:
\begin{theorem}\label{th:nonbinaryising}
  Let $\calD(\cal{W},\theta)$ be an $n$-variable non-binary Ising model with
  alphabet size $k$, width $\lambda$ that is $\eta$-identifiable. There exists an algorithm that given $\spty$, $\eta,\rho  \in (0,1)$, and $N = O(\exp(O(\lambda))/\eta^4) \cdot (\log (n/\rho \eta))
  \cdot k^3$ independent samples $Z^1,\ldots,Z^N \leftarrow \calD(A,\theta)$ recovers the underlying dependency graph of $\calD(\cal{W},\theta)$ with probability at least $1-\rho$. The run-time of the algorithm is $O(k^2 n^2 N)$. Moreover, the algorithm can be run in an online manner.
\end{theorem}

At a high-level, the proof proceeds as follows. Let $Z \sim \calD(\cal{W},\theta)$. For each $i \in [n]$, a partial assignment $x \in [k]^{[n] \setminus \{i\}}$, and for each $\alpha \neq \beta \in [k]$,
$$\pr\sbkets{Z_i = \beta \mid Z_i \in \{\alpha,\beta\} \land Z_{-i} = x} = \sigma \bkets{ \bkets{\sum_{j\neq i}
    W_{ij}(\alpha,x_j) - W_{ij}(\beta,x_{j})} + \theta_{i}(\alpha) - \theta_i(\beta)}.$$
    
This is analogous to \eref{eq:condising} and allows us to use an argument similar to that used in \tref{th:ising}. To this end, we first show an analogue of \emph{unbiasedness} for non-binary Ising models. We say a distribution $\calD$ on $[k]^n$ is $\delta$-unbiased if for any $i \in [n]$, and any partial assignment $x$ to $(X_j: j \neq i)$,
$\min_{a} ( \pr[X_i = a  | X_{-i} = x]) \geq \delta$. Just as for Ising models, non-binary Ising models $\calD(\cal{W},\theta)$ also turn out to be $\delta$-unbiased for $\delta$ depending on the width of the distribution. 
\begin{lemma}\label{lm:unbiasedness}
A non-binary Ising distribution $\calD(\cal{W},\theta)$ on $[k]^n$ is $\delta$-unbiased for $\delta = (e^{-2 \lambda}/k)$ where $\lambda$ is the width of $\calD(\cal{W},\theta)$. 
\end{lemma}
\begin{proof}
Let $X \sim \calD(\cal{W},\theta)$. Note that for any $\alpha, \beta \in [k]$, $i \in [n]$, and a partial assignment $x \in [k]^{[n] \setminus \{i\}}$,
$$\frac{\pr[X_i = \alpha | X_{-i} = x]}{\pr[X_i = \beta | X_{-i} = x]} = \frac{ \exp
  \bkets{\sum_{j \neq i} W_{ij}(\alpha,x_{j}) + \theta_i(\alpha)}} {\exp
  \bkets{\sum_{j\neq i} W_{ij}(\beta,x_{j})+ \theta_i(\beta)}} \leq \exp(-2 \lambda).$$
  Therefore, $\min_\alpha \pr[X_i = \alpha | X_{-i} = x] \geq \exp(-2 \lambda)/k$.
\end{proof}

We are now ready to prove \tref{th:nonbinaryising}. 
\begin{proof}[Proof of \tref{th:nonbinaryising}]
Just as in the proof of \tref{th:ising}, the starting point is the following connection to learning GLMs. Let $Z \sim \calD(\cal{W},\theta)$. Then, for any $\alpha,\beta \in [k]$, $i \in [n]$, and a partial assignment $x \in [k]^{[n] \setminus \{i\}}$, 

$$ \frac{\pr \sbkets{Z_{i} = \alpha \mid Z_{-i} = x}}{\pr
  \sbkets{Z_{i} = \beta \mid Z_{-i} = x}} = \frac{ \exp
  \bkets{\sum_{j\neq i} W_{ij}(\alpha,x_{j}) + \theta_i(\alpha)}} {\exp
  \bkets{\sum_{j\neq i} W_{ij}(\beta,x_{j})+ \theta_i(\beta)}}.$$

Thus, $$ \frac{\pr\sbkets{Z_i = \beta \mid  Z_{-i} = x}}{\pr
  \sbkets{Z_i = \alpha \mid  Z_{-1} = x} + \pr \sbkets{Z_{i} = \beta \mid
    Z_{-1} = x}} =   \sigma \bkets{ \bkets{\sum_{j\neq i}
    W_{ij}(\alpha,x_j) - W_{ij}(\beta,x_{j})} + \theta_{i}(\alpha) - \theta_i(\beta)}.$$

In the following we show how to recover the edges of the dependency graph involving vertex $n$; we can repeat the argument for each of the other vertices to infer the graph. 

Fix $\alpha,\beta \in [k]$. We will show how to approximate $W_{ni}(\alpha,b) - W_{ni}(\beta,b)$ for all $i \in [n-1]$ and $b \in [k]$. Once we have such an approximation, we can recover the values $W_{ni}(\alpha,b)$ as
$$W_{ni}(\alpha,b) = \sum_\beta (W_{ni}(\alpha,b) - W_{ni}(\beta,b)).$$

To this end, let $Z' \equiv Z \mid Z_n \in \{\alpha,\beta\}$. Note that we can obtain independent draws from $Z'$ by restricting to examples where $Z_{n}$ takes only values in $\{\alpha, \beta\}$. This will not cost us too much in sample complexity as by \lref{lm:unbiasedness} $\pr[Z_n \in \{\alpha,\beta\}] \geq 2\exp(-2\lambda)/k$.

Let $e: [k] \mapsto \{0,1\}^{k}$ be a function that maps an alphabet
symbol $a$ to a string of length $k$ that is all zeros except for a
$1$ in position $a$.  Then for a string $x \in [k]^{n-1}$, let $\bmx$ be
a string of length $k(n-1)$ equal to $(e(x_1), \ldots, e(x_{n-1}))$.  Let $\bm{w}$
be a string broken into $n$ blocks each of size $k$ where the $b$'th value of block $i$ denoted $\bm{w}_i(b)$ is $W_{ni}(\alpha,b) - W_{ni}(\beta,b)$. Then, from the above arguments, we have
$$\pr[Z'_n = \beta \mid Z_{-n} = x] = \sigma \bkets{ \bkets{\sum_{i < n}
    W_{ni}(\alpha,x_i) - W_{ni}(\beta,x_{i})} + \theta_{i}(\alpha) - \theta_i(\beta)} = \sigma\bkets{\iprod{\bm{w}}{\bmx} + \theta'},$$
    where $\theta' = \theta_i(\alpha) - \theta_i(\beta)$. 
    
 Note that as the matrices $W_{ij}$ are centered, we have $\sum_{a \in [k]} \bm{w}_{j}(a) = 0$ for all $j < n$. Let $X$ be the distribution of $\widetilde{\bm{z_{-n}}}$ for $z$ drawn from $Z'$ and let $Y = 1$ if $Z_n = \beta$ and $0$ if $Z_n = \alpha$. Then, 
 $$\ex[Y|X] = \sigma(\iprod{\bm{w}}{X} + \theta').$$
 
 Now, for some parameter $\gamma,\rho' \in (0,1)$ to be chosen later, assume we ran our Sparsitron algorithm as in \tref{th:pconceptmwu} to obtain a vector $\bm{u} \in \R^{(n-1) k}$ and $\theta'' \in \R$ such that with probability at least $1-\rho'$,

$$\ex \sbkets{\bkets{\sigma( \iprod{\bm{w}}{X}+ \theta' ) -
    \sigma( \iprod{\bm{u}}{X} + \theta'' )}^2} \leq
\gamma.$$  

Here, we also assume that for each $i \in [n-1]$, $\sum_{a \in [k]} \bm{u}_i(a) = 0$; for if not, we can set $\bm{u}'_i(a) = \bm{u}_i(a) - (1/k) \sum_b \bm{u}_i(b)$ and $\theta''' = \theta'' + (1/k) \sum_i \sum_b \bm{u}_i(b)$. Next, analogous to \lref{lm:l2toexact}, we show that the above condition implies that $\|\bm{w}- \bm{u}\|_\infty \ll 1$. 
\begin{claim}\label{cl:binexact}
Let $Y$ be a $\delta$-unbiased distribution on $[k]^{n-1}$ and let $\bm{w} = (\bm{w}_i(a): i < n, a \in [k]),\bm{u} = (\bm{u}_i(a): i < n, a \in [k])$ be \emph{centered} (i.e., for every $i$, $\sum_a \bm{w}_i(a) = \sum_a \bm{u}_i(a) = 0$) and for some $\theta, \theta' \in \R$,  $\gamma < e^{- 2\|\bm{w}\|_1 - 2|\theta'| -6} \delta$, 
$$\ex \sbkets{\bkets{\sigma( \iprod{\bm{w}}{\widetilde{\bm{Y}}}+ \theta' ) -
    \sigma( \iprod{\bm{u}}{\widetilde{\bm{Y}}} + \theta'' )}^2} \leq
\gamma.$$
Then, $\|\bm{w} - \bm{u}|\|_{\infty} \leq O(1) \cdot
e^{\|\bm{w}\|_1 + |\theta'|} \sqrt{\gamma / \delta}.$
\end{claim}

Indeed, from the above claim and by \lref{lm:unbiasedness}, we get that for every $i < n$, $a \in [k]$,
$$\abs{W_{ni}(\alpha,a) - W_{ni}(\beta,a) - \bm{u}_i(a)} < O(1) e^{O(\lambda)} \sqrt{k \gamma}.$$

The above argument was for a fixed $\alpha, \beta \in [k]$. Repeating the argument for all $\alpha, \beta \in [k]$, we find estimates $U^{\alpha,\beta}_i(a)$ for $i < n$ and $a \in [k]$ such that
$$\abs{W_{ni}(\alpha,a) - W_{ni}(\beta,a) - U^{\alpha,\beta}_i(a)} < O(1) e^{O(\lambda)} \sqrt{k \gamma}.$$

Now, set $U_{ni}(\alpha,a) = (1/k) \sum_{\beta} U^{\alpha,\beta}_i(a)$. Then, by the above inequality and the fact that $\sum_\beta W_{ni}(\beta,a) = 0$, we get that for every $i < n$ and $\alpha, a \in [k]$, 
$$\abs{W_{ni}(\alpha,a) - U_{ni}(\alpha,a)} < O(1) e^{O(\lambda)} \sqrt{k \gamma} < \eta/2$$
 if we set $\gamma = \eta^2 e^{-C \lambda}/C k$ for a sufficiently big constant $C$. Now, given the above, we can identify the neighbors of vertex $n$ as follows: For each $i < n$ if $\max_{\alpha, a} \abs{U_{ni}(\alpha,a)} > \eta/2$, $i$ is a neighbor of $n$. This works as $\cal{D}(\cal{W},\theta)$ is $\eta$-identifiable. 
 
It remains to bound the number of samples needed. As per \tref{th:pconceptmwu}, the number of samples of $Z'$ needed is 
$$O(\lambda^2 \log(n/\rho' \gamma)/\gamma^2 = O(1) k^2 \cdot e^{O(\lambda)} \log(n/\rho' \eta) = N'.$$

Finally, as $Z$ is $\delta$-unbiased for $\delta = e^{-2 \lambda}/k$, to get the above number of samples of $Z'$ with probability $1-\rho/n^2k^2$, it suffices to take 
$$O(N'/\delta) = O(1) k^3 e^{O(\lambda)} \log(n \cdot n^2 k^2/\rho \eta)$$
samples. Finally, we can take a union bound over vertices $i \in [n]$ and $\alpha, \beta \in [k]$. This completes the proof of the theorem.
\ignore{ 

We first show how to recover the neighborhood of vertex $n$ assuming the above claim. First, for each pair $\alpha \neq \beta \in [k]$, proceeding as above, we can

, we enumerate over
  all pairs of alphabet symbols $\alpha, \beta$ and run {\em
    Sparsitron} separately for each pair.

  Notice that for each pair of alphabet symbols $\alpha, \beta$, the
  conditional distribution \\ $\pr \sbkets{X_1 = \beta \mid X_{-1}
    \land X_1 \in \{\alpha, \beta\}}$ is $(1/2)
  e^{-2\lambda}$-unbiased.  Setting $\epsilon = \eta/2$ and then
  $\gamma$ as in the proof of Theorem \ref{th:ising}, we can apply
  Lemma \ref{lem:nb} to obtain estimates of $W_{1,j}(\alpha,i) -
  W_{1j}(\beta,i)$ for all $\alpha,\beta,i,j$.  If we find $\alpha,
  \beta$ such that $|W_{1,j}(\alpha, i) - W_{1,j}(\beta,i)| > \eta/2$
  then we can conclude that $\{i,j\}$ is an edge in the graph.  The
  bounds on sample complexity and running time are obtained as in the
  proof of Theorem \ref{th:ising}.

  We note that we can also obtain estimates for $W_{1j}(\alpha,i)$ for all
  $\alpha,i$ by noticing that $$\sum_{\beta} \bkets{
W_{1j}(\alpha,i) - W_{1j}(\beta,i)} = k \cdot W_{1j}(\alpha,i) -
\sum_{\beta}W_{1j}(\beta,i) = k \cdot W_{1j}(\alpha,i)$$ (recall from Fact \ref{factnb} we can assume
that the column sums of $W$ are without loss of generality zero). }


\end{proof}

\begin{proof}[Proof of \clref{cl:binexact}]

\ignore{
\begin{lemma} \label{lem:nb}
If $x$ is $\delta$-unbiased and $\ex \sbkets{\bkets{\sigma( \langle \bmw, \bmx \rangle + \theta' ) -
    \sigma( \langle \bm{u}, \bmx \rangle + \theta }^2} \leq
\epsilon$ then $\|\bmw - \bm{u}|\|_{\infty} \leq O(1) \cdot
e^{\|\bmw\|_1 + |\theta'|} \sqrt{\epsilon / \delta}.$
\end{lemma}}
\newcommand{\bmy}{\widetilde{y}}
For $y \in [k]^{n-1}$ define $p(y) = \langle \bm{w}, \bmy \rangle + \theta'$ and
$q(y) = \langle \bm{u}, \bmy \rangle + \theta''$.   Let $y^{i,a}$ be
the string obtained from $y$ by setting the $i$th coordinate of
$y$ to $a$. Fix an index $i < n$. Then, as in the binary case, by \clref{clm:sigmoidlb} and the fact that
$Y$ is drawn from a $\delta$-unbiased distribution, we have
\begin{align} \label{eqn:main}
\ex_{Y_{i} | Y_{-i}} \sbkets{ \bkets{ \sigma(p(Y)) - \sigma(q(Y))}^2} \geq \delta
\cdot e^{-\|\bmw\|_1 - |\theta'| - 3} \cdot \min \bkets{1, \max_{a}
  |p(Y^{i,a}) - q(Y^{i,a})|^2}.
\end{align}

Since we assumed that the sum of the values in each block of $\bm{w}$
are zero, we have that
\begin{align} \label{eqn1}
\sum_{b} \bkets{p(y^{i,a}) - p(y^{i,b})} & =  \sum_{b}
\bkets{\bm{w}_{i}(a) - \bm{w}_{i}(b)} = k \bm{w}_{i}(a). 
\end{align}

and similarly

\begin{align} \label{eqn2}
\sum_{b} \bkets{q(y^{i,a}) - q(y^{i,b})} & =  \sum_{b}
\bkets{\bm{u}_{i}(a) - \bm{u}_{i}(b)} = k \bm{u}_{i}(a). 
\end{align}

Subtracting equation \ref{eqn2} from equation \ref{eqn1} and applying
the triangle inequality we obtain

\begin{align}
k \cdot | \bkets{\bm{w}_{i}(a) - \bm{u}_{i}(a)} | \leq & \sum_{b} |
p(y^{i,a}) - q(y^{i,a})| + |p(y^{i,b}) - q(y^{i,b})| \\
 \leq & ~ 2k \cdot \max_{b \in [k]} |p(y^{i,b}) - q(y^{i,b})|.
\end{align}

Thus we have $\max_{b \in [k]} |p(y^{i,b}) - q(y^{i,b})| \geq 1/2
\cdot \max_{a \in [k]} |\bm{w}_{i}(a) - \bm{u}_{i}(a)|$.   We also have,
 that $\ex_{Y} \sbkets{ \bkets{ \sigma(p(Y)) - \sigma(q(Y))}^2} \leq
\gamma$.  We therefore must have 

$$ \max_{a \in [k]} |\bm{w}_{i}(a) - \bm{u}_{i}(a)| \leq O(1) \cdot
\sqrt{\gamma / \delta} \cdot e^{\|\bm{w}\|_1 + |\theta'|}. $$
\end{proof}





\bibliographystyle{alpha}
\bibliography{MRF}

\appendix

\section{Hardness of Learning $t$-wise Markov random fields}
Bresler, Gamarnik, and Shah \cite{BGS} showed how to embed parity learning as a Markov random field and showed that a restricted class of algorithms must take time $n^{\Omega(d)}$ to learn degree MRFs defined on degree $d$ graphs.  Here we use the same construction, and for completeness we give a proof.  The conclusion is that, assuming the hardness of learning sparse parities with noise, for degree $d$ graphs, learning $t$-wise MRFs ($t < d$) will require time $n^{\Omega(t)}$.   Our positive results match this conditional lower bound.

\begin{itemize}

\item Let $\chi_{S}:
  \{-1,1\}^n \rightarrow \{-1,1\}^n$ be an unknown parity function on a
  subset $S$, $|S| \leq k$, of $n$ inputs bits (i.e., $f(x) = \prod_{i \in S} x_{i})$. Let ${\cal C}_k$ be the concept class of all parity functions on subsets $S$ of size at most $k$. Let $\D$ be the uniform
  distribution on $\{-1,1\}^n$.  

 \item Fix an unknown $c \in {\cal C}_k$ and consider the following random experiment: $x$ is drawn according to ${\cal D}$ and with probability $1/2 + \eta$ (for some constant $\eta$), the tuple $(x, c(x))$ is output.  With probability $1/2 - \eta$, the tuple $(x,c(x)')$ is output where $c'(x)$ is the complement of $c(x)$.

\item   The $k$-LSPN problem is as follows: Given i.i.d. such tuples as described above, find $h$ such that $\Pr_{x}[h(x) \neq c(x)] \leq\epsilon$.  
 \end{itemize}

 Now we reduce $k$-LSPN to learning the graph structure of a $(k+1)$-wise Markov random field.  Let $S$ denote the $k$ indices of the unknown parity function Let $G$ be a graph on $n+1$ vertices $X_{1},\ldots,X_{n},Y$ equal to a clique on the set of vertices corresponding to set $S$ and vertex $Y$.  Consider the probability distribution

$$ \pr[Z = (x_{1}, \ldots, x_{n}, y)] \propto \exp\bkets{\gamma \chi_{S}(x) y} $$

for some constant $\gamma$. A case analysis shows that $p_1 = \pr[Y = \chi_{S}(X)] \propto e^{\gamma}$ and $p_2 = \pr [Y \neq \chi_{S}(X)] \propto e^{-\gamma}$.  Hence the ratio $p_1/p_2$ is approximately $1 + 2\gamma$.  Since $p_1 + p_2 = 1$, by choosing $\gamma$ to be a sufficiently small, we will have $p_1 \geq 1/2 + \eta$ and $p_2 \leq 1/2 - \eta$ for some small (but constant) $\eta$. 

Further, it is easy to see that the parity of any subset of $X_{i}$'s is unbiased.  By the Vazirani XOR lemma \cite{Goldreich}, this implies that the $X_{i}$s are uniformly distributed.  Therefore, the distribution encoded by this $(k+1)$-wise MRF is precisely the distribution described in the $k$-LSPN problem.  If we could discover the underlying clique in the Markov random field, we would be able to learn the underlying sparse parity.  Hence, learning $k+1$-MRFs is harder than $k$-LSPN.

The current best algorithm for learning $k$-LSPN is due to Valiant \cite{Valiant15} and runs in time $n^{\Omega(0.8k)}$. Any algorithm running in time $n^{o(k)}$ would be a major breakthrough in theoretical computer science.

\end{document}